\documentclass[a4paper,10pt]{amsart}

\usepackage[a4paper]{geometry}
\usepackage{amsfonts}
\usepackage{amsmath}
\usepackage{amssymb}
\usepackage{amsthm}
\usepackage{amsbsy}
\usepackage{xcolor}
\usepackage{xspace}

\usepackage{graphicx}
\usepackage{color}
\usepackage{listings}
\usepackage{algorithm}
\usepackage{algpseudocode}
\usepackage{caption}
\usepackage{subcaption}
\usepackage{float}
\usepackage{enumerate}
\usepackage{multicol}
\usepackage{mathrsfs}
\usepackage[all,cmtip]{xy}
\usepackage[color=blue!20,textsize=footnotesize]{todonotes}
\usepackage{tikz}
\usepackage{soul}
\usepackage{pgfplots}
\usepackage[font=small]{caption}
\usepackage{apptools}

\usetikzlibrary{external,calc,fit,patterns,decorations.markings,matrix,3d,arrows.meta}

\definecolor{gray1}{RGB}{240, 240, 240}
\definecolor{gray2}{RGB}{200, 200, 200}
\definecolor{gray3}{RGB}{160, 160, 160}
\definecolor{gray4}{RGB}{140, 140, 140}
\definecolor{gray5}{RGB}{120, 120, 120}

\newcounter{dummy} \numberwithin{dummy}{subsection}
\newtheorem{theorem}[dummy]{Theorem}

\newtheorem{definition}[dummy]{Definition}

\theoremstyle{remark}
\newtheorem{remark}[dummy]{Remark}

\newcommand{\calH}{\mathcal{H}}
\newcommand{\calL}{\mathcal{L}}

\newcommand{\frakX}{\mathfrak{X}}

\DeclareMathOperator{\SE}{SE}
\DeclareMathOperator{\SO}{SO}

\DeclareMathOperator{\spn}{span}

\DeclareMathOperator{\length}{length}
\newcommand{\PTRT}{PT\mathbb{R}^2}

\newcommand{\WaxOn}{\texttt{WaxOn}\xspace}
\newcommand{\WaxOff}{\texttt{WaxOff}\xspace}

\numberwithin{equation}{section}

\begin{document}
\let\WriteBookmarks\relax
\def\floatpagepagefraction{1}
\def\textpagefraction{.001}

\thanks{The authors are supported by the grant GeoProCo from the Trond Mohn Foundation - Grant TMS2021STG02 (GeoProCo). Part of the results of this paper has appeared in the first author's master thesis.}

\title[Geometry of the visual cortex with applications]{Geometry of the visual cortex with applications to image inpainting and enhancement}
\author[F. Ballerin and E. Grong]{Francesco Ballerin and Erlend Grong}
\email{francesco.ballerin@uib.no}  
\email{erlend.grong@uib.no}  
\date{}




\pgfplotsset{compat=1.18}

\emergencystretch 3em

\begin{abstract}
    Equipping the rototranslation group $\SE(2)$ with a sub-Riemannian structure inspired by the visual cortex V1, we propose algorithms for image inpainting and enhancement based on hypoelliptic diffusion. We innovate on previous implementations of the methods by Citti, Sarti, and Boscain et al., by proposing an alternative that prevents fading and is capable of producing sharper results in a procedure that we call \WaxOn-\WaxOff. We also exploit the sub-Riemannian structure to define a completely new unsharp filter using $\SE(2)$, analogous to the classical unsharp filter for 2D image processing. We demonstrate our method on blood vessels enhancement in retinal scans.
\end{abstract}

\maketitle

\section{Introduction}\label{section_introduction}

Image inpainting is a process that aims at restoring information that has been lost in a region of the canvas. 
Although in recent times neural networks have proven to be extremely effective at most digital image processing tasks, the black box structure of most implementations makes such tools difficult to understand and trust, as well as prone to unexpected failures. Moreover, the computational requirements to train such networks can often be prohibitive.

The aim of this work is to introduce a new robust and effective algorithm for image inpainting and image enhancement, which does not require training on an image dataset, based on a well-known sub-Riemannian model of the visual cortex V1.

This geometric structure was formalized by Petitot in 1998 \cite{Petitot1999Mar, Petitot2008, Petitot2017} and exploited by Citti and Sarti \cite{Citti2006-xa, Citti2016-gz} and Boscain et al. \cite{Boscain2009, Boscain2012-se,Boscain2014-nd, Boscain2018-et, Boscain2018-qd, boscain2019introduction} to derive a biologically inspired image restoration algorithm.
The algorithm is blind, i.e. it does not exploit information on the position of the corruption. The main idea of this approach is that by lifting the image from $\mathbb{R}^2$ to $\SE(2)$ restoration can be achieved by hypoelliptic diffusion in the direction of the lifted level curves of the image.

The main drawback of this family of algorithms is that, by design, the resulting images incur heavy blurring and loss of higher-frequency information. We aim at tackling such problems by exploiting the sub-Riemannian structure on $\SE(2)$, and the vector field that is transversal to the lifted level lines of the image, to produce sharpening tools. By alternating diffusion along level curves (\WaxOn) with concentration transversally to the level curves and in the direction of the gradient (\WaxOff), we are able to produce inpainting in the damaged areas while also preserving sharpness.

In addition to image inpainting, transversal diffusion allows us to design a completely new sharpening filter, analogous to the classic 2D unsharp filter, with applications to image enhancement and preprocessing. We demonstrate our new method by enhancing blood vessels in a retinal scan. By combining this unsharp filter with the a new procedure that we call \WaxOff, we also propose an inpainting method that preserves even more high-frequency details, and allows for longer diffusion times without excessive blurring.

Section \ref{geometry_vision} briefly introduces the geometrical preliminaries needed to understand the sub-Riemannian model of the visual cortex V1. For more details on sub-Riemannian geometry, we refer to Appendix A. In section \ref{sub_application_image_processing} we present the previous work by Citti and Sarti \cite{Citti2006-xa} and Boscain et al. \cite{Boscain2012-se} on exploiting the sub-Riemannian geometry of the V1 model for image restoration. In section \ref{new_innovations} we present our new developments on the subject, which consist of a new approach to treat images in the geometrical model of interest and two approaches based on sharpening techniques to address blurring, one being the \WaxOff-procedure and the other the unsharp filter on $SE(2)$. We also present a Python package that has been developed in conjunction with this article, from which all figures of this work are derived. \footnote{Python package containing code and reproducible Jupyter Notebooks can be found at \texttt{https://github.com/ballerin/v1diffusion}.}

We thank Xavier Pennec for the helpful discussions around the \WaxOff procedure and Francesco Rossi for sharing with us a MATLAB implementation of the algorithm, which was an inspiration for our own implementation.

\subsection*{Related work}
Extensive work has been presented in the study of the Lie Group $\SE(2)$ relative to image processing, with special attention to the field of retinal imagery. We give in this section an overview of related works.

Citti and Sarti \cite{Citti2006-xa,Citti2016-gz} have presented a framework based on mean curvature flow to achieve restoration through diffusion by performing perception completion in the space $\SE(2)$ followed by non-maxima suppression in the same space.
Work by Franken and Duits has been done to study the possible ways to enhance and sharpen the effects of the proposed algorithms \cite{Franken2009, Duits2010}.
Boscain et al. have expanded on these methods \cite{Boscain2012-se} and have deepened the discussion by proposing a semidiscrete treatment of the problem as well as a variation of the problem in which the position of the corruption is known \cite{Boscain2014-nd, Boscain2018-et, Boscain2018-qd}. 
Other works have focused on the specific application of retinal vessel analysis exploiting the geometric structure of $\SE(2)$ as orientation scores, for example in \cite{Hannink2014, Zhang2016Aug, Bekkers2014Jul}. 
In addition the Lie Group $\SE(2)$ has been successfully used in the context of geometric deep learning \cite{Bekkers2018Apr}, introducing an $\SE(2)$ group convolution layer that allows for a state-of-the-art performance without the need for data augmentation in problems within histopathology, retinal imaging, and electron microscopy.
\section{Sub-Riemannian geometry and image processing}\label{geometry_vision}

\subsection{Geometry of the rototranslation group} \label{sec:SE2}
We consider \emph{the special euclidean group} $\SE(2)$ of dimension 2, as the group of matrices
\begin{align} \label{SE2} \SE(2)
&=\left\{ B  =  \left.
		\begin{bmatrix}
			A & \mathbf{x} \\
			0&1
		\end{bmatrix} \; \right|\; \begin{array}{c}\mathbf{x}\in\mathbb{R}^2, \\  A\in \SO(2) \end{array} \right\} 
=\left\{ \left. 
		\begin{bmatrix}
			\cos\theta&-\sin\theta&x\\
			\sin\theta&\cos\theta&y\\
			0&0&1\\
		\end{bmatrix} \; \right|\; \begin{array}{c} x,y\in\mathbb{R}, \\ \theta\in \mathbb{R}/2\pi \mathbb{Z} \end{array} \right\} 
\end{align}
which is a matrix Lie group, i.e. both a group under matrix multiplication and a smooth manifold. $\SE(2)$ is the group of all transformations on $\mathbb{R}^2$ which preserve distances, orientations, and angles (rigid transformations). Any such transformation can be written as a rototranslation
\begin{alignat*}{4}
    T: &\mathbb{R}^2&\rightarrow & \mathbb{R}^2 \qquad \qquad && A\in \SO(2) \\
    &\textbf{a} &\mapsto & A\textbf{a}+\textbf{x} \qquad \qquad && \mathbf{x},\mathbf{a} \in \mathbb{R}^2.
\end{alignat*}
Using the coordinates $(x,y,\theta)$ in \eqref{SE2}, we see that $\SE(2)$ as a space can be identified with the $3$-dimensional cylinder $\mathbb{R}^2 \times S^1$.

A vector field $X$ on $\SE(2)$ is called \emph{left invariant} if it has the property that any curve $\gamma(t)$ is tangent to $X$ if and only if $B \gamma(t)$ is tangent to $X$ for any $B \in \SE(2)$. Any such vector field will be a linear combination of vector fields $X_1$, $X_2$ and $X_3$, written in the coordinates $(x,y,\theta)$ as
\begin{equation} \label{Xbasis} X_1 = \cos(\theta) \partial_x + \sin(\theta) \partial_y, \qquad X_2 = \partial_\theta, \qquad X_3 = -\sin(\theta)\partial_x + \cos(\theta)\partial_y. \end{equation}
These vector fields have flows $e^{tC_j}(B) = B \exp(tC_j)$, where
$$
C_1 = \begin{bmatrix} 0& 0 & 1\\ 0& 0&0\\ 0&0&0\end{bmatrix}, \qquad
C_2 = \begin{bmatrix} 0& -1 & 0\\ 1& 0&0\\ 0&0&0\end{bmatrix}, \qquad
C_3 = \begin{bmatrix} 0& 0 & 0\\ 0& 0& 1\\ 0&0&0\end{bmatrix}.
$$
Here we have used $e^{tX}$ for the flow of the vector field $X$, while $\exp(C) = \sum_{j=0}^\infty \frac{C^j}{j!}$ denotes the usual exponential of matrix $C$.

Recall that \emph{Lie bracket of vector fields} $X$ and $Y$ is the vector field $[X,Y]$ defined by $[X,Y]f = X(Yf) - Y(Xf)$. We note that the Lie bracket of left-invariant vector fields is always left-invariant, and for the basis in \eqref{Xbasis}, we obtain
$$[X_1, X_2] = - X_3, \qquad [X_2, X_3] = X_1, \qquad [X_1,X_3] =0.$$
We observe in particular that we can obtain $X_3$ by combining $X_1$ and $X_2$. We will take advantage of this in what follows.
By letting $\mathcal H = \spn\{X_1, X_2\}$ we obtain a bracket-generating distribution, which guarantees under Chow-Rashevski\"i's theorem that any two points in $\SE(2)$ can be connected by a path whose tangent vectors lie in the distribution $\mathcal H$. See Appendix A for details. Similarly, if we consider a horizontal distribution $\tilde \calH = \spn \{ X_2, X_3\}$, then we can obtain $X_1$ as the bracket $[X_2,X_3] = X_1$ so that this distribution is also bracket-generating. We will use both for our proposed algorithm.

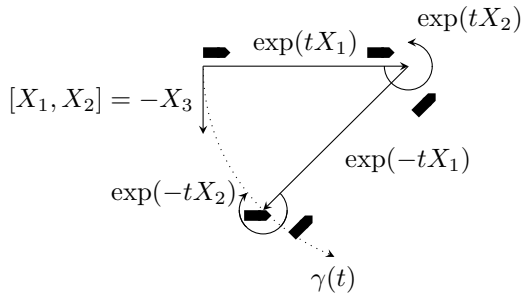
\begin{figure}[H]
	\centering
	\begin{tikzpicture} [scale=0.9]
		\draw [-stealth](0,0) -- (3,0) node[midway,above] {$\exp(tX_1)$};
		\draw [-stealth](3,0) -- (3-2.1213,-2.1213) node[midway,below right] {$\exp(-tX_1)$};
		
		\begin{scope}
			\clip (0,0) rectangle (1.9,-3);
			\draw[dotted] (3,0) circle (3cm);
		\end{scope}
		\draw [-stealth](1.9,-2.8) -- (1.92,-2.81) node[midway,below] {$\gamma(t)$};
		
		\begin{scope}
			\clip (2,-1) rectangle (4,0);
			\draw[-] (3,0) circle (0.35cm);
		\end{scope}
		\begin{scope}
			\clip (3,0) rectangle (4,1);
			\draw[-] (3,0) circle (0.35cm);
		\end{scope}
		\draw [-stealth](3,0.35) -- (2.99,0.35) node[midway,above right] {$\exp(tX_2)$};
		
		\begin{scope}
			\clip (-2.1213-1,-2.1213-1) rectangle (3-2.1213+1,-2.1213);
			\draw[-] (3-2.1213,-2.1213) circle (0.35cm);
		\end{scope}
		\begin{scope}
			\clip (3-2.1213+0.35,-2.1213) circle (0.27cm);
			\draw[-] (3-2.1213,-2.1213) circle (0.35cm);
		\end{scope}
		\begin{scope}
			\clip (3-2.1213-0.35,-2.1213) circle (0.27cm);
			\draw[-] (3-2.1213,-2.1213) circle (0.35cm);
		\end{scope}
		\draw [-stealth](3-2.1213-0.245,-2.1213+0.245) -- (3-2.1213-0.235,-2.1213+0.255) node[midway,left] {$\exp(-tX_2)$};
		
		\draw[-{Triangle Cap},line width=4pt] (0,0.2) -- (0.4,0.2);
		\draw[-{Triangle Cap},line width=4pt] (2.4,0.2) -- (2.8,0.2);
		\draw[-{Triangle Cap},line width=4pt] (3.1,-0.7) -- (3.4,-0.4);
		\draw[-{Triangle Cap},line width=4pt] (1.3,-2.5) -- (1.6,-2.2);
		\draw[-{Triangle Cap},line width=4pt] (0.6,-2.2) -- (1.0,-2.2);
		
		\draw [-stealth](0,0) -- (0,-1) node[midway,left] {$[X_1,X_2]=-X_3$};
	\end{tikzpicture}
	\;\;\;\;\;\;\;\;\;\;\;\;\;\;\;\;\;
	\caption{Modeling $\SE(2)$ as a car with orientation, where $X_1$ is forward movement and $X_2$ is counter-clockwise rotation. Translation in the direction $X_3$ can be obtained by combining infinitesimal movements along $X_1$ and $X_2$.} 
\end{figure}

\subsection{The visual cortex V1 and curve completion}\label{cortex_preliminaries}
When seeking to produce image processing algorithms, one option is to draw inspiration from the biological model that drives human perception.
Visual information is processed in the brain by the \textit{visual cortex}, located in the occipital lobe.
\emph{The primary visual cortex V1} is higy specialized in processing orientations and recognize patterns.
From neurology we know that a neuron switches from its dormant state to its excited state when it gets sufficiently stimulated by either an external input or by other neurons \cite{OrhanE1969}.

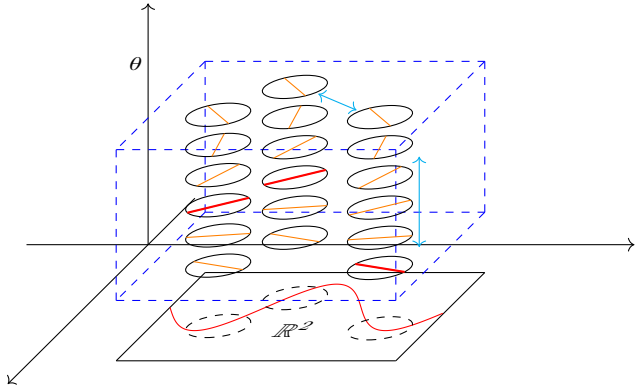
\begin{figure}[H] 
	\centering
	\begin{tikzpicture}[scale=0.8]
		
		\draw[thin,->] (-2,0) -- (8,0);
		\draw[thin,->] (0,0) -- (0,4);
		
		\begin{scope}[canvas is zx plane at y=0]
			\draw[->] (-2,0) -- (6,0);
			
			\draw[-] (1.2,1.4) -- (5,1.4);
			\draw[-] (1.2,6) -- (5,6);
			\draw[-] (1.2,1.4) -- (1.2,6);
			\draw[-] (5,1.4) -- (5,6);
			
			\draw[samples=400,variable=\t, smooth, domain=0:(pi*1.46), color=red] plot ({2.7+sin(deg(2*\t))} , 1.4+\t);
			
			\draw[dashed] (3.5,2.5) circle (0.5);
			\draw[dashed] (2.3,3.3) circle (0.5);
			\draw[dashed] (3.6,5.2) circle (0.5);
		\end{scope}
		
		\begin{scope}[canvas is zx plane at y=1]
			\draw[-, color=blue, dashed] (1.2,1.4) -- (5,1.4);
			\draw[-, color=blue, dashed] (1.2,6) -- (5,6);
			\draw[-, color=blue, dashed] (1.2,1.4) -- (1.2,6);
			\draw[-, color=blue, dashed] (5,1.4) -- (5,6);
		\end{scope}
		
		\begin{scope}[canvas is zx plane at y=3.5]
			\draw[-, color=blue, dashed] (1.2,1.4) -- (5,1.4);
			\draw[-, color=blue, dashed] (1.2,6) -- (5,6);
			\draw[-, color=blue, dashed] (1.2,1.4) -- (1.2,6);
			\draw[-, color=blue, dashed] (5,1.4) -- (5,6);
		\end{scope}
		
		\begin{scope}[canvas is xy plane at z=1.2]
			\draw[-, color=blue, dashed] (1.4,1) -- (1.4,3.5);
			\draw[-, color=blue, dashed] (6,1) -- (6,3.5);
		\end{scope}
		
		\begin{scope}[canvas is xy plane at z=5]
			\draw[-, color=blue, dashed] (1.4,1) -- (1.4,3.5);
			\draw[-, color=blue, dashed] (6,1) -- (6,3.5);
		\end{scope}

		\begin{scope}[every node/.append style={
				yslant=0,sloped}
			];
			\node at (-.2,3) {\scalebox{1}[.7]{$\theta$}};
		\end{scope}
		
		\foreach \theta in {-20,10,...,130}{
			\begin{scope}[canvas is zx plane at y=1.5+((\theta-10)/60)]
				\draw[-] (3.5,2.5) circle (0.5);
				\draw[-] (2.3,3.3) circle (0.5);
				\draw[-] (3.6,5.2) circle (0.5);
				
				\draw[-,color=orange] ({3.5+0.5*sin(\theta)},{2.5-0.5*cos(\theta)}) -- ({3.5-0.5*sin(\theta)},{2.5+0.5*cos(\theta)});
				\draw[-,color=orange] ({2.3+0.5*sin(\theta)},{3.3-0.5*cos(\theta)}) -- ({2.3-0.5*sin(\theta)},{3.3+0.5*cos(\theta)});
				\draw[-,color=orange] ({3.6+0.5*sin(\theta)},{5.2-0.5*cos(\theta)}) -- ({3.6-0.5*sin(\theta)},{5.2+0.5*cos(\theta)});
			\end{scope}
		}
		\begin{scope}[canvas is zx plane at y=1]
			\draw[-,color=red, thick] ({3.6+0.5*sin(-20)},{5.2-0.5*cos(-20)}) -- ({3.6-0.5*sin(-20)},{5.2+0.5*cos(-20)});
		\end{scope}
		
		\begin{scope}[canvas is zx plane at y=2]
			\draw[-,color=red, thick] ({3.5+0.5*sin(40)},{2.5-0.5*cos(40)}) -- ({3.5-0.5*sin(40)},{2.5+0.5*cos(40)});
			\draw[-,color=red, thick] ({2.3+0.5*sin(40)},{3.3-0.5*cos(40)}) -- ({2.3-0.5*sin(40)},{3.3+0.5*cos(40)});
		\end{scope}

		\begin{scope}[canvas is xy plane at z=4]
			\draw[<->, color=cyan] (6,1.5) -- (6,3);
		\end{scope}
		\begin{scope}[canvas is zx plane at y=3.5]
			\draw[<->, color=cyan] (2.6,3.8) -- (3.3,4.7);
		\end{scope}
		
		\begin{scope}[every node/.append style={
				xslant=1,sloped}
			];
			\node at (2.4,-1.4) {\scalebox{1}[.7]{$\mathbb{R}^2$}};
		\end{scope}
	\end{tikzpicture}
	\caption{Visual Cortex V1 under a stimulus (red curve): the red orientation columns receive direct stimulus from the input, as opposed to the orange ones. Excitatory synapses for simple cells located in the same hypercolumn or that are spatially close and sensitive to the same orientation are indicated by cyan arrows.}
	\label{fig:V1}
\end{figure}

Upon studying the visual cortex V1 one finds that the neurons are arranged in cells with elongated receptive fields, which exhibit even or odd symmetric patterns similar to Gabor filters~\cite{Marcelja1980Nov}.
In a simplified model, the neurons inside V1 are grouped into \textit{orientation columns}, each being sensitive to stimuli at a specific point of the retina, corresponding to the spatial coordinate on the field of view, and a specific orientation.
Orientation columns are themselves grouped together into \textit{hypercolumns} that are sensitive to stimuli in a certain position of the retina, regardless of the orientation.
Orientation columns are connected in two different ways: \textit{vertical} (inhibitory) synapses and \textit{horizontal} (excitatory) synapses. The vertical connections happen between columns belonging to the same hypercolumn, whereas horizontal ones happen between columns belonging to different hypercolumns that are spatially close and have similar orientation sensitivity. See Figure~\ref{fig:V1} for an illustration.

We can model V1  as $\SE(2)$, where the hypercolumns are given as coordinates $(x,y)$ and the orientation sensitivity of the orientation columns is given by $\theta$. Because of the inhibitory synapses, we are only allowed to move along directions $X_2$ (within a hypercolumn) and $X_1$ (between hypercolumns).

As presented in \cite{Petitot2003Mar} the map $SE(2)\to\mathbb R^2$ associating to each neuron of V1 its preferred orientation presents three classes of qualitatively different points: regular points, pinwheels, and saddle points. In particular, pinwheels are singular points where all orientations converge while saddle points are singular points where the orientations bifurcate.

We remark that the orientations mentioned are directionless, and so the correct space would be  $PT\mathbb{R}^2:=SE(2)/\simeq$
which is the result of the identification $(x,y,\theta)\simeq(x,y,\theta+\pi)$. The notation reflects that $PT\mathbb{R}^2$ can be considered as the space of lines in the tangent space. See e.g. \cite{Boscain2010Sep} for details. However, we can continue to develop our theory on $\SE(2)$ as long as we are using operations that are invariant under the identification $\simeq$. This has the advantage that we can use global formulas for the vector fields $X_1$, $X_2$, $X_3$, though we note that $X_1^2$, $X_2^2$ and $X_3^2$ are invariant under quotient by $\simeq$.

Gestalt laws have been proposed to explain the phenomenon in which the human brain ``fills in'' the gaps between curves or edges that present similar orientations while enhancing the contrast of objects that present different orientations.
For a more mathematical description, let $\gamma_0:[a,b]\cup [c,d]\rightarrow \mathbb{R}^2$,  $a<b<c<d$, be a smooth curve, parametrized by arc length, that is partially hidden in the interval $t\in(b,c)$. We want to find a curve $\gamma:[b,c]\rightarrow \mathbb{R}^2$, parametrized by arc length that completes $\gamma_0$ while minimizing some cost $E[\gamma]$. We require that  $\gamma(b)=\gamma_0(b)$, $\gamma(c)=\gamma_0(c)$ and for initial and final derivatives,  $\dot\gamma(b)= \pm\dot\gamma_0(b)$ and $\dot\gamma(c) = \pm \dot\gamma_0(c)$. We are looking for a curve that is as smooth as possible, in the sense that the curve needs to minimize the energy $E_\beta(\gamma) = \int_b^c (1+ \beta \lvert K_\gamma(s)\rvert^2)ds$, $\beta >0$, with the geodesic curvature given by $K_\gamma(t) = \frac{\dot x \ddot y - \dot y \ddot x}{(\dot x ^2 + \dot y ^2)^\frac{3}{2}}$.

Operationally we can find such curves by lifting the problem to $\SE(2)$ and considering a curve $\Gamma(t) = (\gamma(t), \theta(t))$ in $\SE(2)$ with $\dot \gamma(b) = \pm (\cos \theta(b) \partial_x + \sin \theta(b)\partial_y)$ and $\dot \gamma(c) = \pm (\cos \theta(c) \partial_x + \sin \theta(c)\partial_y )$. Curves minimizing the energy $E_\beta(\gamma)$ can then be considered as projections of sub-Riemannian geodesics in $\SE(2)$. To give more details, let $\calH = \spn\{ X_1, X_2\}\subsetneq T\SE(2)$ be as in Section~\ref{sec:SE2}. We introduce a smoothly varying inner product $g_\beta = \langle \cdot , \cdot \rangle_\beta$ defined just on $\calH$ by identities
$$\langle X_1, X_1 \rangle_\beta =1, \qquad \langle X_1, X_2 \rangle_\beta = 0, \qquad \langle X_2, X_2 \rangle_\beta = \beta^{-1}.$$
This fiber metric $g_\beta$ on $\calH$ is then called \emph{a sub-Riemannian metric}, and the pair $(\calH,g_\beta)$ will be \emph{a sub-Riemannian structure}. Geodesics of such a sub-Riemannian metric are then curves $\Gamma(t)$ that are tangent to $\calH$ and minimize the length with respect to $g_\beta$,
$$\length_\beta(\Gamma) = \int_b^c \|\dot {\Gamma}(t)\| \, dt = \int_b^c \sqrt{\langle X_1, \dot{\Gamma}(t) \rangle_\beta +  \langle X_2,  \dot {\Gamma}(t) \rangle_\beta} \, dt .$$
We remark that the requirement that $\Gamma$ is tangent to $\calH$ is exactly analogous to just moving in the ``admissible directions of V1'', where a smaller $\beta$ indicates an increased cost of moving within a hypercolumn. For a proof that the problem of minimizing the energy $E_\beta$ is equivalent to the sub-Riemannian problem, see \cite{Boscain2010Sep}.

\subsection{Curve completion in image processing} In the concrete applications described in this work, which are in the field of image processing, one will not find curves to complete and functionals to minimize, but rather corrupted images to restore. Images can be thought of in terms of curves if one takes into consideration the level curves of the image function, i.e. non-degenerate connected components of the level sets, as in this case, we need to treat multiple curves at once.

As described in \cite{Boscain2012-se, Prandi2017Apr} one can approach this problem by considering in a stochastic way all possible admissible paths starting at the endpoint of the curve to reconstruct, and model the controls by independent Wiener processes $u_t$ and $v_t$ obtaining the following SDE:
\[\begin{pmatrix}dx_t\\dy_t\\d\theta_t\end{pmatrix} = \sqrt{2} X_1 \circ du_t + \sqrt{2 \beta} X_2 \circ dv_t = \sqrt{2} \begin{pmatrix}\cos\theta_t\\\sin\theta_t\\0\end{pmatrix} \circ du_t + \sqrt{2\beta} \begin{pmatrix}0\\0\\1\end{pmatrix} \circ dv_t\]
The diffusion process associated to such SDE is $\frac{\partial\Psi}{\partial_t} = \Delta_\beta \Psi$
where
\[\Delta_\beta = X_1^2 + \beta X_2^2 = \left(\cos\theta\frac{\partial}{\partial_x}+\sin\theta\frac{\partial}{\partial y}\right)^2 + \beta \frac{\partial^2}{\partial \theta^2}.\]
We remark that $\Delta_\beta$ is symmetric with respect to the volume $d\mu = |dx \wedge dy \wedge d\theta|$ which is the Haar measure on $\SE(2)$. We call $\Delta_\beta$ the \emph{sub-Laplacian} of $g_\beta$. The operator $\Delta_\beta$ is not elliptic, but it has a smooth, strictly positive heat kernel $p_t(x,y)$ with respect to $d\mu$.
In the next section, we will introduce a classic approach to image inpainting exploiting the fact that $\Delta_\beta$ is the intrinsic sub-Laplacian operator of $\SE(2)$ endowed with sub-Riemannian structure $(\calH, g_\beta)$. 

We remark that  curves, due to pinwheels and saddle points, may not always have well defined lifts at every point. An approach to this issue can be through a "blowing up" model as presented in \cite{Petitot2003Mar}, which deals with pinwheel singularities from the sub-Riemannian perspective. However, as we will see in the next section, for practical applications related to image processing this turns out to not be a critical issue if performing a preprocessing step consisting of Guassian smoothing.

\section{Sub-Riemannian application to image processing}\label{sub_application_image_processing}

\subsection{The classical inpainting algorithm}\label{known_algorithm}
The inpainting algorithm proposed by Citti and Sarti \cite{Citti2006-xa} and improved by Boscain, Duplaix, Gauthier and Rossi \cite{Boscain2012-se}, is a ``blind algorithm'', i.e. it makes no assumption on the position of the corruption, and is therefore applied to the whole image without distinction between which areas contain noise and which areas do not. The input of the algorithm is a grayscale image, a signal from a rectangular portion $R \subset \mathbb R ^2$ with values between 0 and 1, where 0 is regarded as white and 1 as black. This is in contrast with the usual convention in image processing and is due to the fact that the details on a photographic picture are most commonly black rather than white. Corruption is represented as an area of constant value 0 (white). This can be changed in order to adapt to the nature of the corruption.

Let $\Pi: \SE(2) \to \mathbb{R}^2$ be the projection. Recall that for a continuously differentiable image $I:\mathbb{R}^2\supset R\rightarrow [0,1]$ the derivative $X_3(I\circ \Pi)$ at $(x,y,\theta)$ takes the form
$$\textstyle X_3(I \circ \Pi)(x,y,\theta)=-\sin(\theta)(\partial_x I)(x,y)+\cos(\theta)(\partial_y I)(x,y),$$
which gives the derivative of $I$ in the direction of the vector $(-\sin\theta,\cos\theta)$. Let us for simplicity write $X_3(I \circ \Pi)$ simply as $X_3 I$. The maximum of $\theta \mapsto |X_3I(x,y,\theta)|^2$ is achieved when $(-\sin\theta,\cos\theta)$ is parallel to the gradient $\nabla I(x,y)$.
If the level curve of the image is defined at a point, it follows that such maximum is achieved when $(\cos\theta, \sin\theta)$ is the direction of the level curve at that point, regardless of orientation. 

Recall that a smooth function $f:\mathbb R^2\to\mathbb R$ is said to be Morse if it has only isolated critical points with nondegenerate Hessian. For such a function the points for which the direction of the level curve is not well defined are isolated. It is a known result that $f$ defined to be $f:=I\ast G_s$, the convolution with a Gaussian of standard deviation $s_x = s_y := s>0$, is generically a Morse function \cite{Boscain2012-se}. From now on let $I$ denote the convolution of the original image by a Gaussian of standard deviation $s>0$ rather than the raw image itself.

Let $O:\SE(2) \to [0,1]$ be defined by 
$$O(x,y,\theta):=\begin{cases}
	I(x,y)& \text{if $\lvert X_3 I\rvert(x,y,\theta) =\max_{\phi} \lvert X_3I\rvert(x,y,\phi)$,}\\
	0&\text{otherwise,}
\end{cases}$$
and mapping a non-critical point $(x,y,\theta)$ to the value of the image at $(x,y)$ if $\theta$ is the direction of the level curve at the point $(x,y)$, regardless of orientation, or zero otherwise. The whole image domain is lifted in $\SE(2)$ on the domain
$$\Sigma_0=\{(x,y,\theta):\lvert X_3I\rvert(x,y,\theta) =\max_{\phi} \lvert X_3I\rvert(x,y,\phi) >0\}$$

\begin{figure}[H]
		\centering
		\begin{subfigure}[t]{0.21\linewidth}
			\centering
			\includegraphics[width=\linewidth]{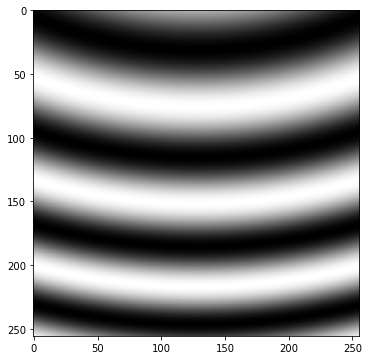}
		\end{subfigure}\;\;
		\begin{subfigure}[t]{0.24\linewidth}
			\centering
			\includegraphics[width=\linewidth]{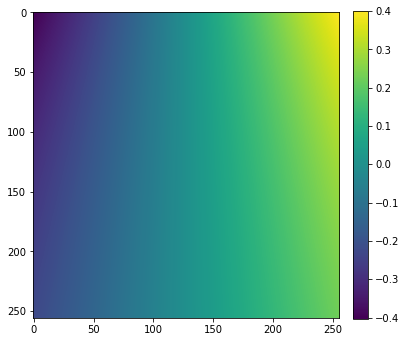}
		\end{subfigure}\;\;
        \begin{subfigure}[t]{0.24\linewidth}
			\centering
			\includegraphics[width=\linewidth]{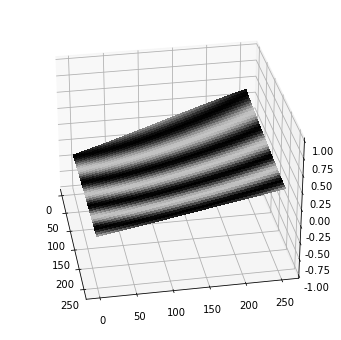}
		\end{subfigure}
        \\
        \begin{subfigure}[t]{0.21\linewidth}
			\centering
			\includegraphics[width=\linewidth]{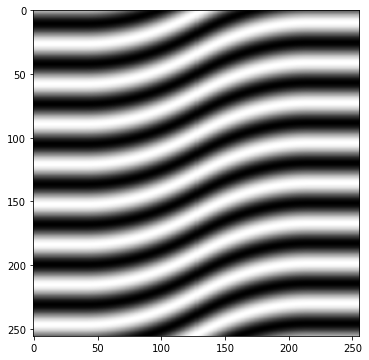}
		\end{subfigure}\;\;
		\begin{subfigure}[t]{0.24\linewidth}
			\centering
			\includegraphics[width=\linewidth]{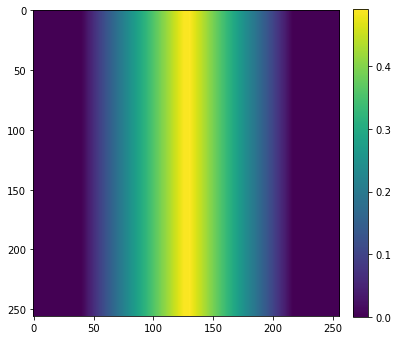}
		\end{subfigure}\;\;
        \begin{subfigure}[t]{0.24\linewidth}
			\centering
			\includegraphics[width=\linewidth]{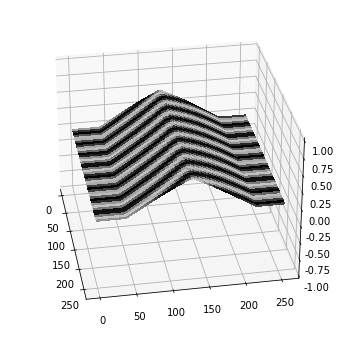}
		\end{subfigure}
    \caption{Examples of lifted images. In the first column are the original images, in the second column the orientation of the level lines, and in the third column, the images lifted to $\SE(2)$, suppressing the trivial zero values for visualization purposes.}
\end{figure}

The lifted set corresponds to the maximum of activity of the output of simple cells that are stimulated exclusively by external signals, which can be modeled mathematically as a Dirac mass concentrated on $\Sigma_0$
\begin{equation} \label{tildeIDelta}
\tilde I(x,y,\theta)=O(x,y,\theta)\delta_{\Sigma_0}.
\end{equation}
$\tilde I (x,y,\theta)$ then corresponds to the image as perceived by our model of V1, where all the neural activity is concentrated on the Dirac mass $\delta_{\Sigma_0}$.
Integrating over a fiber $\Pi^{-1}(x,y) \cong S^1$ that is non-critical at position $(x,y)$ yields exactly $I(x,y)$. If a point happens to be critical for the function $I(x,y)$, then by definition will not be part of $\Sigma_0$ since the gradient at that point vanishes.

The result of the restoration process is given by computing the solution~$u(t,x,y,\theta)$ at time~$T$ of 
\begin{equation} \label{heatflow}\begin{cases}
    \partial_t u = \Delta_\beta u,\\
    u(0,x,y,\theta) = \tilde I (x,y,\theta)
\end{cases}\end{equation}
where $\Delta_\beta= X_1^2 + \beta X_2^2$. The parameter $\beta$ is useful in practice to tune the ``strength'' of the diffusion that is performed in the direction $X_2$ compared to the direction $X_1$. This corresponds to defining how strong are the horizontal and vertical excitatory synapses in the V1 biological model. For details on the original implementation and numerical schemes see \cite{Citti2006-xa, Boscain2012-se}. Our implementation makes use of the naive finite element method for solving the sub-Riemannian heat equation. Results could potentially be further improved by more complex solvers.The reconstructed image is recovered by projecting the solution $u(T,x,y,\theta)$ of the diffusion equation to $\mathbb R^2$, either by integrating over fibers or by taking the maximum.

\begin{figure}[H]
	\centering
	\begin{subfigure}[t]{0.3\linewidth}
		\begin{tikzpicture}[scale=0.6]
			\begin{axis}[axis equal image, restrict x to domain=0:1, restrict y to domain=0:1]
				
				\addplot[variable=t,mesh,domain={-0.35}:{0.35}, color=red, samples=10] ({1/2+t},{1/2});
				\addplot[variable=t,mesh,domain={-0.35}:{0.35}, color=green, samples=10] ({1/2},{1/2+t});
				\foreach \k in {0.6,0.9, ...,1}{
					\addplot[variable=t,mesh,domain=0:{0.3/\k}, color=red, samples=10] ({1/2+\k^2*sin(deg(t))},{1/2+\k^2*(cos(deg(t))-1)});
					\addplot[variable=t,mesh,domain=0:{0.3/\k}, color=red, samples=10] ({1/2+\k^2*sin(deg(t))},{1/2+\k^2*(-cos(deg(t))+1)});
					\addplot[variable=t,mesh,domain=0:{0.3/\k}, color=red, samples=10] ({1/2-\k^2*sin(deg(t))},{1/2+\k^2*(-cos(deg(t))+1)});
					\addplot[variable=t,mesh,domain=0:{0.3/\k}, color=red, samples=10] ({1/2-\k^2*sin(deg(t))},{1/2+\k^2*(cos(deg(t))-1)});
					
					\addplot[variable=t,mesh,domain=0:{0.3/\k}, color=green, samples=10] ({1/2+\k^2*(cos(deg(t))-1)},{1/2+\k^2*sin(deg(t))});
					\addplot[variable=t,mesh,domain=0:{0.3/\k}, color=green, samples=10] ({1/2+\k^2*(-cos(deg(t))+1)}, {1/2+\k^2*sin(deg(t))});
					\addplot[variable=t,mesh,domain=0:{0.3/\k}, color=green, samples=10] ({1/2+\k^2*(-cos(deg(t))+1)}, {1/2-\k^2*sin(deg(t))});
					\addplot[variable=t,mesh,domain=0:{0.3/\k}, color=green, samples=10] ({1/2+\k^2*(cos(deg(t))-1)}, {1/2-\k^2*sin(deg(t))});
				}
				\addplot[variable=t,mesh,domain=0.103:0.897, color=black] (t,{(2*(t-1/2))^3+1/2});
			\end{axis}
		\end{tikzpicture}
		\subcaption{Projections to $\mathbb R^2$ of the integral lines}
	\end{subfigure}
	\begin{subfigure}[t]{0.3\linewidth}
		\begin{tikzpicture}[scale=0.55]
			\begin{axis}
				\addplot3[variable=t,mesh,domain=0.103:0.897, color=black, samples=20] (t,{(2*(t-1/2))^3+1/2}, {atan(6*(2*t-1)^2)/180*3.14});
				
				\addplot3[variable=t,mesh,domain={-0.35}:{0.35}, color=red, samples=10] ({1/2+t},{1/2},0);
				\foreach \k in {0.6,0.9, ...,1}{
					\addplot3[variable=t,mesh,domain=0:{0.3/\k}, color=red, samples=10] ({1/2+\k^2*sin(deg(t))},{1/2+\k^2*(cos(deg(t))-1)},-t);
					\addplot3[variable=t,mesh,domain=0:{0.3/\k}, color=red, samples=10] ({1/2+\k^2*sin(deg(t))},{1/2+\k^2*(-cos(deg(t))+1)},t);
					\addplot3[variable=t,mesh,domain=0:{0.3/\k}, color=red, samples=10] ({1/2-\k^2*sin(deg(t))},{1/2+\k^2*(-cos(deg(t))+1)},-t);
					\addplot3[variable=t,mesh,domain=0:{0.3/\k}, color=red, samples=10] ({1/2-\k^2*sin(deg(t))},{1/2+\k^2*(cos(deg(t))-1)},t);
				}
			\end{axis}
		\end{tikzpicture}
		\subcaption{Integral lines of $X_1^2+\beta X_2^2$ in $\PTRT$}
	\end{subfigure}
	\;\;
	\begin{subfigure}[t]{0.3\linewidth}
		\begin{tikzpicture}[scale=0.55]
			\begin{axis}
				\addplot3[variable=t,mesh,domain=0.103:0.897, color=black, samples=20] (t,{(2*(t-1/2))^3+1/2}, {atan(6*(2*t-1)^2)/180*3.14});
				
				\addplot3[variable=t,mesh,domain={-0.35}:{0.35}, color=green, samples=10] ({1/2},{1/2+t},0);
				\foreach \k in {0.6,0.9, ...,1}{
					\addplot3[variable=t,mesh,domain=0:{0.3/\k}, color=green, samples=10] ({1/2+\k^2*(cos(deg(t))-1)},{1/2+\k^2*sin(deg(t))}, t);
					\addplot3[variable=t,mesh,domain=0:{0.3/\k}, color=green, samples=10] ({1/2+\k^2*(-cos(deg(t))+1)}, {1/2+\k^2*sin(deg(t))}, -t);
					\addplot3[variable=t,mesh,domain=0:{0.3/\k}, color=green, samples=10] ({1/2+\k^2*(-cos(deg(t))+1)}, {1/2-\k^2*sin(deg(t))}, t);
					\addplot3[variable=t,mesh,domain=0:{0.3/\k}, color=green, samples=10] ({1/2+\k^2*(cos(deg(t))-1)}, {1/2-\k^2*sin(deg(t))}, -t);
				}
			\end{axis}
		\end{tikzpicture}
		\subcaption{Integral lines of $X_3^2+\beta X_2^2$ in $\PTRT$}
	\end{subfigure}
	\caption{Integral lines of the vector fields $X_1^2+\beta X_2^2$ (red) and $X_3^2+\beta X_2^2$ (green) for a polynomial curve, at point $\left(\frac{1}{2},\frac{1}{2}\right)$, varying the coefficient $\beta$.} 
\end{figure}
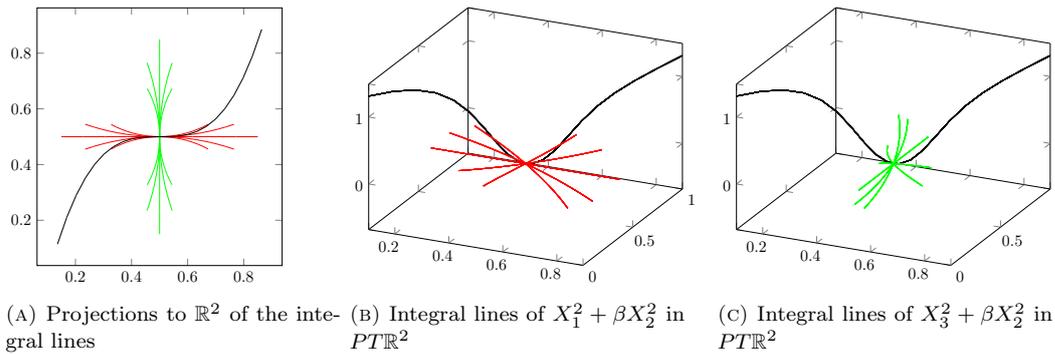

\begin{figure}[H]
    \centering
    \begin{subfigure}[t]{0.23\textwidth}
    \includegraphics[width=\textwidth]{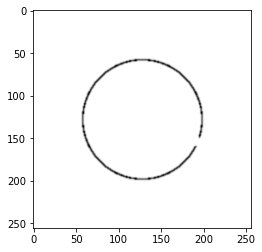}
    \caption{Original image}
    \end{subfigure}
    \begin{subfigure}[t]{0.23\textwidth}
    \includegraphics[width=\textwidth]{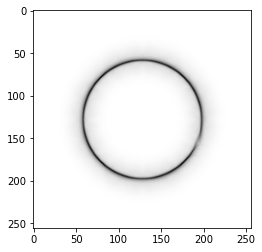}
    \caption{Diffusion at $T=60$ with $\beta=0$}
    \end{subfigure}
    \begin{subfigure}[t]{0.23\textwidth}
    \includegraphics[width=\textwidth]{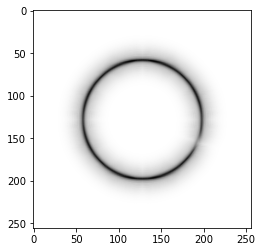}
    \caption{Diffusion at $T=60$ with $\beta=0.25$}
    \end{subfigure}
    \begin{subfigure}[t]{0.23\textwidth}
    \includegraphics[width=\textwidth]{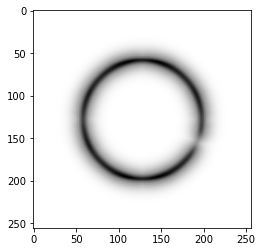}
    \caption{Diffusion at $T=60$ with $\beta=0.5$}
    \end{subfigure}
    \caption{Application of the classic restoration algorithm to a basic example of a broken circle. We see that an increase in $\beta$ produces a more spread-out diffusion.}
    \label{fig:Beta}
\end{figure}
\begin{remark}
One might get the impression from Figure~\ref{fig:Beta}, that we should keep $\beta$ as small as possible or even zero. However, if we want to complete level curves where the endpoints do not have the same orientation, we will need some contributions $X_2$ for the level curves to rotate. 
\end{remark}

\section{New innovations in the sub-Riemannian method}\label{new_innovations}

\subsection{Lift as a normal distribution}\label{gaussian_lift}
In the work by Marcelja \cite{Marcelja1980Nov} and Jones and Palmer \cite{Jones1987Dec} the similarity in behavior between simple cells and Gabor filters is studied and presented. This was later confirmed in the works by Olshausen and Field \cite{Field1994Jul, Olshausen1996Jun, Olshausen1997Dec} who studied sparse codes for natural images and their relation to Gabor-wavelets and simple cells in the cortex V1. For a Gabor filter with a fixed orientation $\theta$ the output of a signal through the filter decays exponentially as the angle of the original signal differs from $\theta$. It is then argued that each hypercolumn behaves as a stack of Gabor filters with different orientations. From this idea of neural signal in the visual cortex V1 decaying exponentially as the angle in the fiber $S^1$ differs from the angle of the level curve, we can model each fiber as a normal distribution centered at the angle corresponding to that of the level curve. This choice preserves locality but without the need to specify a kernel size for the Gabor filters. We are therefore ``spreading'' the input signal around the orientation of maximum response $\theta$ of the simple cells, and doing so following a Gaussian distribution centered around such orientation with this procedure
\begin{equation} \label{Lsigma} \tilde I = \mathcal{L}_{\sigma}(I) = (I \circ \Pi) \cdot \exp \left(- \frac{(X_1 (I \circ \Pi))^2}{2 \sigma^2 |\nabla I|^2} \right) = (I \circ \Pi) \cdot \exp \left(- \frac{|\nabla I|^2 -(X_3 (I \circ \Pi))^2}{2 \sigma^2 |\nabla I|^2} \right) .\end{equation}

In other words, $I(x,y)\cdot \exp\left(-\left\langle \frac{\nabla I}{\lvert\nabla I \rvert} , (\cos\theta, \sin\theta)\right\rangle^2/(2\sigma^2)\right)$, where we are using the $\sigma^2$ to adjust the variance around the optimal angle. The lift in \eqref{tildeIDelta}
can be considered as a limiting case of \eqref{Lsigma} when~$\sigma \to 0$. 

One can alternatively model such lift in terms of a wrapped normal distribution, or the more tractable von Mises distribution.

Having defined such a lift $I \to \tilde I$ , we need a projection process inverting this lift. The simplest projection that is inverse to $\mathcal{L}_\sigma$ is the $\Pi_{\max}$ is given by the maximum
$$\Pi_{\max}(\tilde I)(x,y) = \max_{\theta} \tilde I(x,y,\theta).$$
However, we do not in general get a smooth function as a result of taking the maximum. An alternative projection can be defined as follows.
\begin{theorem}
Define an operator $\Pi_{\sigma}: C^\infty(\SE(2),(0,1]) \to C^\infty(\mathbb{R}^2, (0,1])$ by
$$\Pi_{\sigma}(\tilde I)(x,y) = \exp\left(\frac{1}{4\sigma} + \frac{1}{2\pi} \int_0^{2\pi} \ln \tilde I(x,y,\theta) d\theta \right).$$
Then $\Pi_{\sigma}(\mathcal{L}_{\sigma}(I)) = I$.
\end{theorem}
\begin{proof}
If $\tilde I = \mathcal{L}_{\sigma}(I)$, then averaging over a period yields
\begin{align*} \int_0^{2\pi} \ln(\tilde I) d\theta
& = \int_0^{2\pi} \left( \ln I(x,y) - \frac{(\cos \theta \partial_x I+ \sin \theta \partial_y I)^2}{2\sigma|\nabla I|^2}\right) d\theta = 2\pi \left( \ln I(x,y) - \frac{1}{4\sigma} \right) 
\end{align*}
We hence recover the original image $I$ from this procedure. 
\end{proof}

The maximum $\Pi_{\max}$ will in general be more computationally efficient, and our experiments show that problems of non-differentiability is not a large issue in practice. It also coincides with the projection used in earlier literature. It is however susceptible to fading under diffusion for small values of $\sigma$.

Let $u(t,x,y,\theta) = e^{t\Delta_\beta} \tilde I(x,y,\theta)$ be the solution of \eqref{heatflow}. Our main advantage by introducing the lifting $\calL_{\sigma}$ is that we prevent noise and fading of the image under the sub-Riemannian heat flow $e^{t\Delta_\beta}$.

\begin{figure}[H]
    \centering
    \begin{subfigure}[t]{0.3\linewidth}
        \centering
        \includegraphics[width=\linewidth,  clip=true,trim=10px 10px 10px 10px]{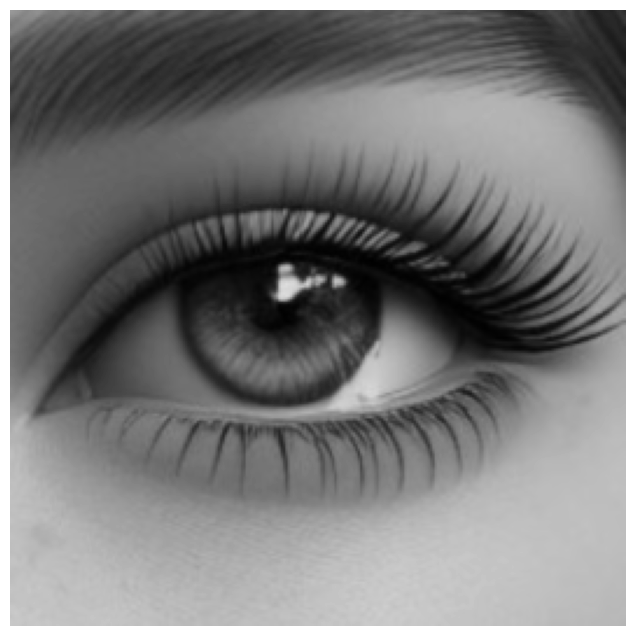}
        \caption{Original image, AI generated}
    \end{subfigure}\;
    \begin{subfigure}[t]{0.3\linewidth}
        \centering
        \includegraphics[width=\linewidth,  clip=true,trim=10px 10px 10px 10px]{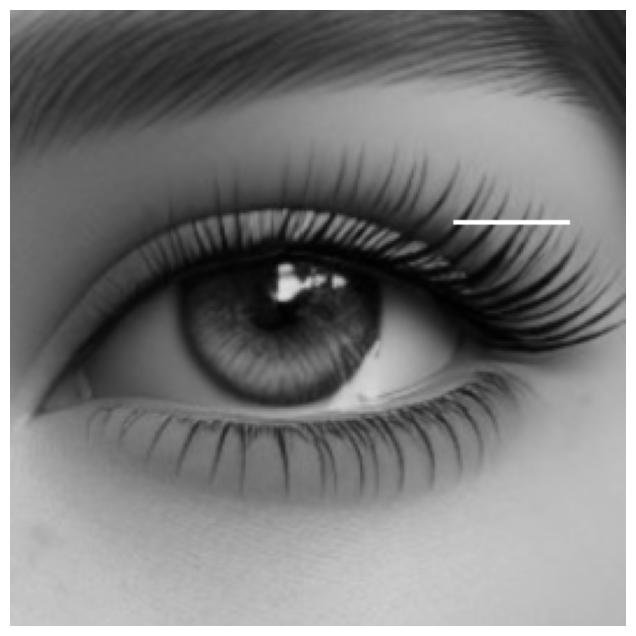}
        \caption{Corrupted image}
    \end{subfigure}\;
    \begin{subfigure}[t]{0.3\linewidth}
        \centering
        \includegraphics[width=\linewidth,  clip=true,trim=10px 10px 10px 10px]{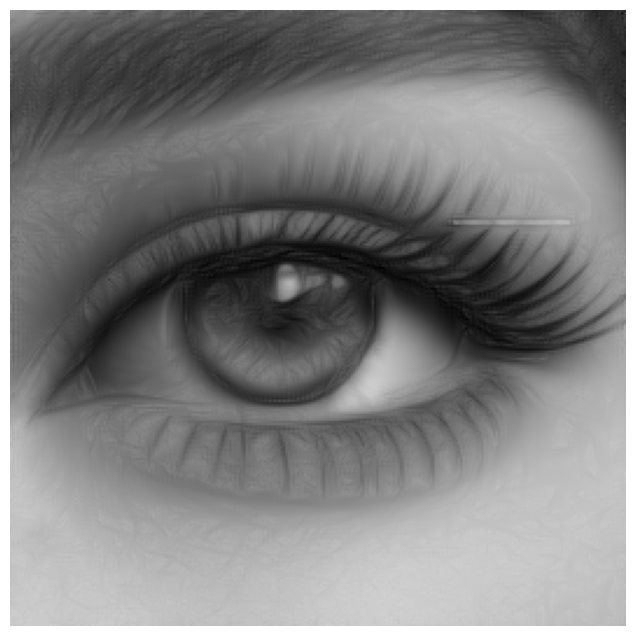}
        \caption{$\sigma = 0.8$, $T=10$, and $\beta=0.25$}
    \end{subfigure}
    \begin{subfigure}[t]{0.3\linewidth}
        \centering
        \includegraphics[width=\linewidth,  clip=true,trim=10px 10px 10px 10px]{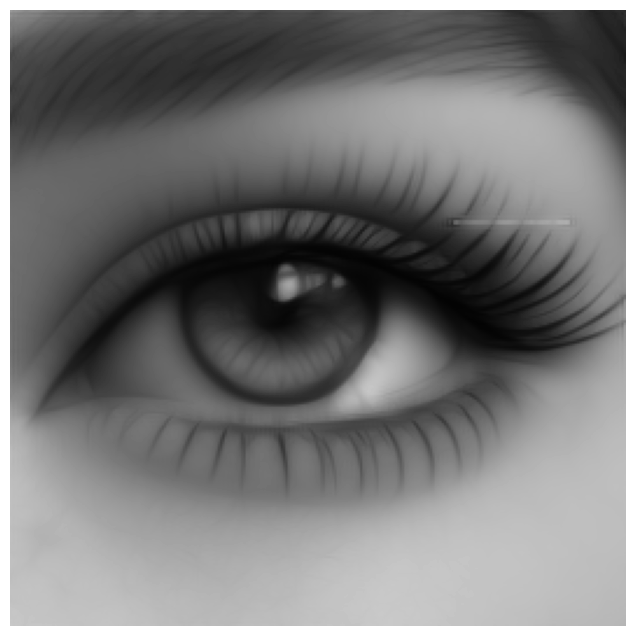}
        \caption{$\sigma = 5$, $T=10$, and $\beta=0.25$}
    \end{subfigure}\;
    \begin{subfigure}[t]{0.3\linewidth}
        \centering
        \includegraphics[width=\linewidth,  clip=true,trim=10px 10px 10px 10px]{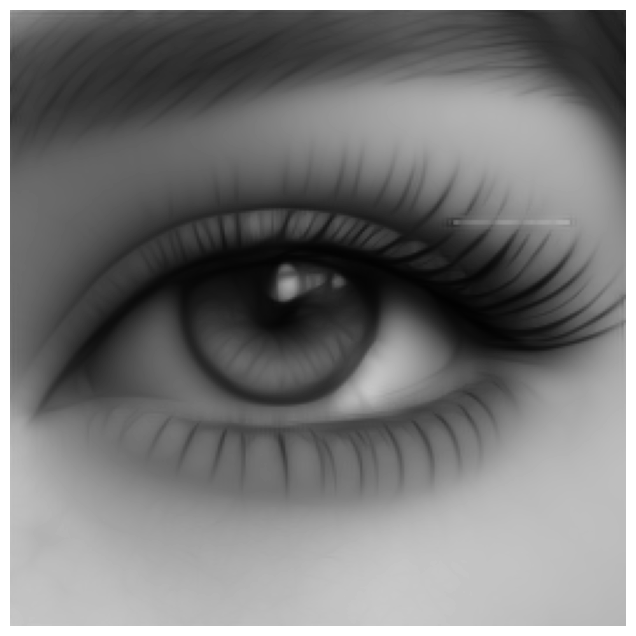}
        \caption{$\sigma = 100$, $T=10$, and $\beta=0.25$}
    \end{subfigure}\;
    \begin{subfigure}[t]{0.3\linewidth}
        \centering
        \includegraphics[width=\linewidth,  clip=true,trim=10px 10px 10px 10px]{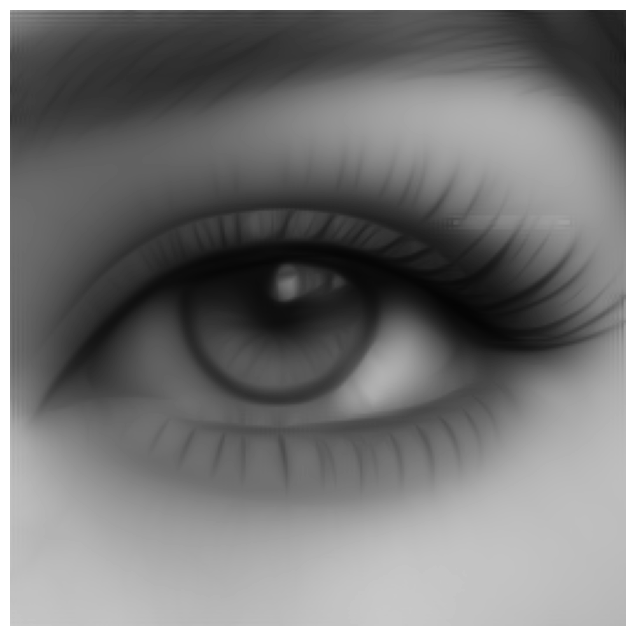}
        \caption{$\sigma = 100$, $T=30$, and $\beta=0.25$}
    \end{subfigure}
    \caption{The original image (a) is lifted with Gaussian lift with different values for $\sigma$. We see that in (c), we quickly incur into fading and noise for small values of $\sigma$. This effect reduces for larger $\sigma$ (d,e,f), even when considering a longer time scale (f).}
    \label{fig:Sigma}
\end{figure}
We see that the resulting function after applying lift in $\calL_\sigma(I)$ is $\pi$-periodic. It is hence sufficent to consider $\theta \in [0,\pi)$, which we can consider as working on $PT\mathbb{R}^2 = \SE(2)/\simeq$. The projection $\Pi_{\sigma}$ can then be defined using an average from $0$ to $\pi$ instead.

\subsection{Preserving details through \WaxOn, \WaxOff}\label{wax_on_wax_off}

As we expect from a diffusion method, the result will not only inpaint along level curves but also blur the image. We can also see this from Figure~\ref{fig:Sigma}~(D). In order to obtain sharper images after a diffusion, we will need to concentrate our image in the direction transverse to our level curves, that is, in the direction of $X_3$. We therefore consider a second sub-Riemannian structure $\{ \tilde \calH, \tilde g_\beta\}$ on $\SE(2)$, such that $X_3$ and $\frac{1}{\sqrt{\beta}} X_2$ forms an orthonormal basis. Let $\tilde \Delta_\beta$ be the operator
\begin{equation} \label{tildeDelta} \tilde \Delta_\beta = X_3^2 + \beta X_2^2 \end{equation}
Using the lift $\calL_\sigma$ as in \eqref{Lsigma}, diffusion $e^{t\tilde \Delta_\beta}$ with $\tilde \Delta_\beta$ corresponds to diffusing transverse to level curves. The effect we are looking for is a reversing of this diffusion.

\begin{figure}[H]
    \centering
    \begin{subfigure}[t]{0.23\linewidth}
        \centering
        \includegraphics[width=\linewidth,  clip=true,trim=10px 10px 10px 10px]{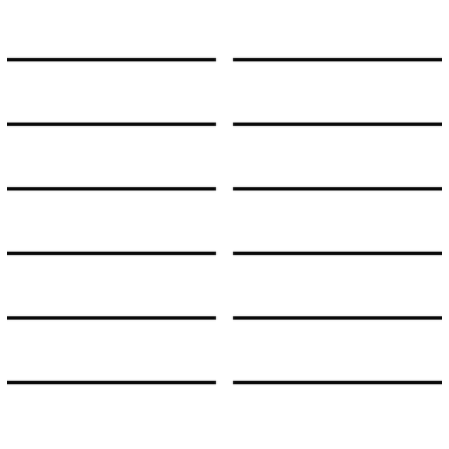}
        \caption{Original image}
    \end{subfigure}\;
    \begin{subfigure}[t]{0.23\linewidth}
        \centering
        \includegraphics[width=\linewidth,  clip=true,trim=10px 10px 10px 10px]{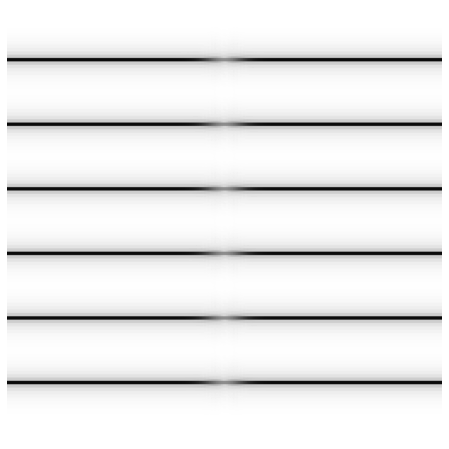}
        \caption{$\Delta_0 = X_1^2$}
    \end{subfigure}\;
    \begin{subfigure}[t]{0.23\linewidth}
        \centering
        \includegraphics[width=\linewidth,  clip=true,trim=10px 10px 10px 10px]{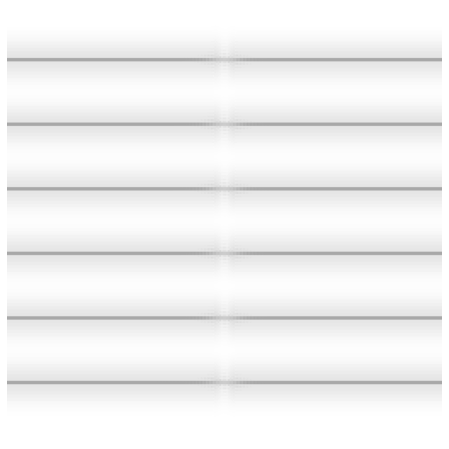}
        \caption{$\tilde \Delta_0 = X_3^2$}
    \end{subfigure}\;
    \begin{subfigure}[t]{0.23\linewidth}
        \centering
        \includegraphics[width=\linewidth,  clip=true,trim=10px 10px 10px 10px]{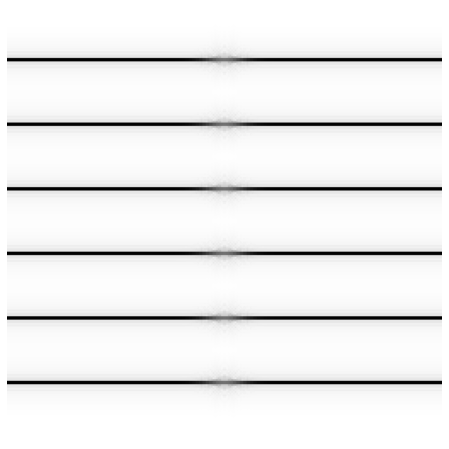}
        \caption{\WaxOn-\WaxOff}
    \end{subfigure}
    \caption{The original image (a) is lifted with Gaussian lift. Diffusion along $X_1$ is applied to obtain (b), while diffusion along $X_3$ yields (c). The \WaxOn-\WaxOff \, algorithm is used to obtain (d), which succeeds in connecting the lines while maintaining the overall image sharp}
\end{figure}

From this argument, we propose a new algorithm for image restoration: assuming that for small $T$ we can recover the initial profile of neural activity by reversing the $\tilde \Delta_\beta$-diffusion PDE, we can sharpen the restored image. Thus we first ``put the wax on'', diffusing the image alongside the level lines, and then ``get the wax off'', sharpening the image along $X_3$. By alternating between applying \WaxOn, that is $e^{t\Delta_\beta}$ for completion of level lines, then applying \WaxOff, that is $e^{-t\tilde \Delta_\beta}$, to concentrate the image on these lines.

\begin{figure}[H]
	\centering
	\begin{subfigure}[t]{0.32\linewidth}
		\begin{tikzpicture}[scale=0.75]
			\begin{axis}[axis equal image, restrict x to domain=0:1, restrict y to domain=0:1]
				\addplot[variable=t,mesh,domain=0.103:0.897, color=white] (t,{(2*(t-1/2))^3+1/2});
				
				\addplot[variable=t,mesh,domain=0.14:0.4, color=black] (t,{(2*(t-1/2))^3+1/2});
				\addplot[variable=t,mesh,domain=0.6:0.86, color=black] (t,{(2*(t-1/2))^3+1/2});
				
				\addplot[variable=t,mesh,domain=0:0.25, color=red, samples=10] ({0.4+t},{0.492+0.24*t});
				\addplot[variable=t,mesh,domain=0:0.25, color=red, samples=10] ({0.6-t},{0.508-0.24*t});
				\foreach \k in {0.6,0.9, ...,1}{
					\addplot[variable=t,mesh,domain=0:{0.25/\k}, color=red, samples=10] ({0.4+\k^2*(sin(deg(t-0.235))+0.232)},{0.492+\k^2*(cos(deg(t-0.235))-1+0.027)});
					\addplot[variable=t,mesh,domain=0:{0.25/\k}, color=red, samples=10] ({0.4+\k^2*(sin(deg(t+0.235))-0.232)},{0.492+\k^2*(-cos(deg(t+0.235))+1-0.027)});
					
					\addplot[variable=t,mesh,domain=0:{0.25/\k}, color=red, samples=10] ({0.6-\k^2*(sin(deg(t-0.235))+0.232)},{0.508-\k^2*(cos(deg(t-0.235))-1+0.027)});
					\addplot[variable=t,mesh,domain=0:{0.25/\k}, color=red, samples=10] ({0.6-\k^2*(sin(deg(t+0.235))-0.232)},{0.508-\k^2*(-cos(deg(t+0.235))+1-0.027)});
				}
				
			\end{axis}
		\end{tikzpicture}
		\subcaption{$\WaxOn$}
	\end{subfigure}\hfil
	\begin{subfigure}[t]{0.32\linewidth}
		\begin{tikzpicture}[scale=0.75]
			\begin{axis}[axis equal image, restrict x to domain=0:1, restrict y to domain=0:1]
				\addplot[variable=t,mesh,domain=0.103:0.897, color=white] (t,{(2*(t-1/2))^3+1/2});

				\addplot[variable=t,mesh,domain={-0.25}:{0.25}, color=green, samples=10] ({1/2},{1/2+t});
				\foreach \k in {0.6,0.9, ...,1}{
					\addplot[variable=t,mesh,domain=0:{0.25/\k}, color=green, samples=10] ({1/2+\k^2*(cos(deg(t))-1)},{1/2+\k^2*sin(deg(t))});
					\addplot[variable=t,mesh,domain=0:{0.25/\k}, color=green, samples=10] ({1/2+\k^2*(-cos(deg(t))+1)}, {1/2+\k^2*sin(deg(t))});
					\addplot[variable=t,mesh,domain=0:{0.25/\k}, color=green, samples=10] ({1/2+\k^2*(-cos(deg(t))+1)}, {1/2-\k^2*sin(deg(t))});
					\addplot[variable=t,mesh,domain=0:{0.25/\k}, color=green, samples=10] ({1/2+\k^2*(cos(deg(t))-1)}, {1/2-\k^2*sin(deg(t))});
				}
				
				\addplot[variable=t,mesh,domain=0.351:0.649, color=gray1] (t,{(2*(t-1/2))^3+1/2 + 0.09*(exp(-(1/(1-((0.5-t)*2/0.3)^2))))});
				\addplot[variable=t,mesh,domain=0.351:0.649, color=gray2] (t,{(2*(t-1/2))^3+1/2 + 0.06*(exp(-(1/(1-((0.5-t)*2/0.3)^2))))});
				\addplot[variable=t,mesh,domain=0.351:0.649, color=gray3] (t,{(2*(t-1/2))^3+1/2 + 0.03*(exp(-(1/(1-((0.5-t)*2/0.3)^2))))});
				
				\addplot[variable=t,mesh,domain=0.351:0.649, color=gray1] (t,{(2*(t-1/2))^3+1/2 - 0.09*(exp(-(1/(1-((0.5-t)*2/0.3)^2))))});
				\addplot[variable=t,mesh,domain=0.351:0.649, color=gray2] (t,{(2*(t-1/2))^3+1/2 - 0.06*(exp(-(1/(1-((0.5-t)*2/0.3)^2))))});
				\addplot[variable=t,mesh,domain=0.351:0.649, color=gray3] (t,{(2*(t-1/2))^3+1/2 - 0.03*(exp(-(1/(1-((0.5-t)*2/0.3)^2))))});
				
				\addplot[variable=t,mesh,domain=0.35:0.65, color=gray5] (t,{(2*(t-1/2))^3+1/2});

				\addplot[variable=t,mesh,domain=0.14:0.35, color=black] (t,{(2*(t-1/2))^3+1/2});
				\addplot[variable=t,mesh,domain=0.65:0.86, color=black] (t,{(2*(t-1/2))^3+1/2});
			\end{axis}
		\end{tikzpicture}
		\subcaption{$\WaxOff$}
	\end{subfigure}\hfil
	\begin{subfigure}[t]{0.32\linewidth}
		\begin{tikzpicture}[scale=0.75]
			\begin{axis}[axis equal image, restrict x to domain=0:1, restrict y to domain=0:1]
				\addplot[variable=t,mesh,domain=0.103:0.897, color=white] (t,{(2*(t-1/2))^3+1/2});
				
				\addplot[variable=t,mesh,domain=0.351:0.649, color=gray1] (t,{(2*(t-1/2))^3+1/2 + 0.045*(exp(-(1/(1-((0.5-t)*2/0.3)^2))))});
				\addplot[variable=t,mesh,domain=0.351:0.649, color=gray2] (t,{(2*(t-1/2))^3+1/2 + 0.03*(exp(-(1/(1-((0.5-t)*2/0.3)^2))))});
				\addplot[variable=t,mesh,domain=0.351:0.649, color=gray3] (t,{(2*(t-1/2))^3+1/2 + 0.015*(exp(-(1/(1-((0.5-t)*2/0.3)^2))))});
				
				\addplot[variable=t,mesh,domain=0.351:0.649, color=gray1] (t,{(2*(t-1/2))^3+1/2 - 0.045*(exp(-(1/(1-((0.5-t)*2/0.3)^2))))});
				\addplot[variable=t,mesh,domain=0.351:0.649, color=gray2] (t,{(2*(t-1/2))^3+1/2 - 0.03*(exp(-(1/(1-((0.5-t)*2/0.3)^2))))});
				\addplot[variable=t,mesh,domain=0.351:0.649, color=gray3] (t,{(2*(t-1/2))^3+1/2 - 0.015*(exp(-(1/(1-((0.5-t)*2/0.3)^2))))});
				
				\addplot[variable=t,mesh,domain=0.35:0.65, color=gray5] (t,{(2*(t-1/2))^3+1/2});
				
				\addplot[variable=t,mesh,domain=0.14:0.35, color=black] (t,{(2*(t-1/2))^3+1/2});
				\addplot[variable=t,mesh,domain=0.65:0.86, color=black] (t,{(2*(t-1/2))^3+1/2});
			\end{axis}
		\end{tikzpicture}
		\subcaption{Final result}
	\end{subfigure}
	\caption{Sketch of intuition behind \WaxOn-\WaxOff} 
\end{figure}
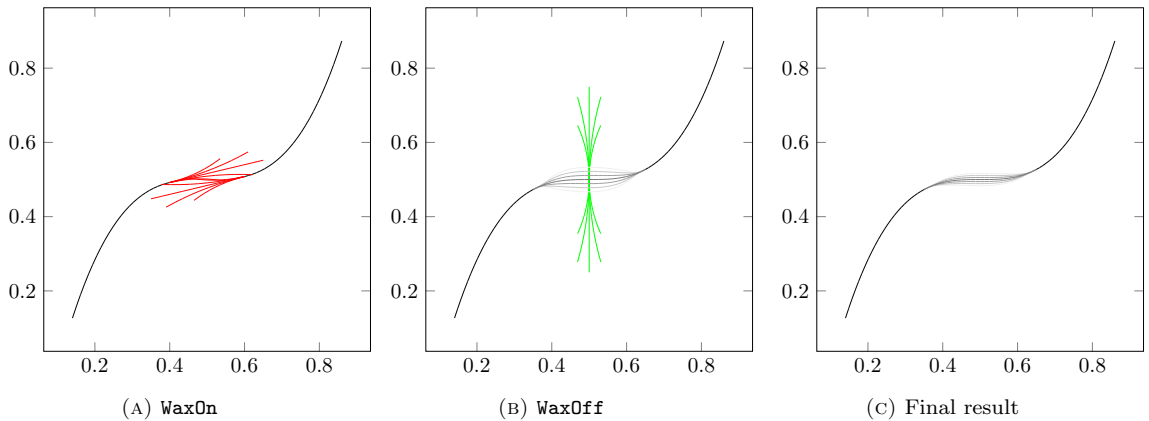

Although we are formally writing $e^{-t\tilde \Delta_\beta}$, the heat flow is an irreversible process and the solution of the heat equation for $t<0$ is not well defined (\cite{Evans2010}, \cite{Strauss2007}).
This is an ill-posed problem as we in general do not have stability with respect to initial data. See e.g. \cite{weber1981analysis} and \cite[Chapter~8.2]{kabanikhin2011inverse} for details. 
Reversing the effect of the heat flow is susceptible to noise and will eventually diverge ``blowing up''. In practice, the \WaxOff needs to be run only up to a small time $T_2$ before any blowup happens.

By applying the two steps repeatedly, the \WaxOn portion of the algorithm yields restoration while the \WaxOff portion sharpens the result while retaining restored information.

Our rule of thumb from experiments is that running \WaxOn for a time $T$, followed by applying \WaxOff up until a time $T_2 = T/8$ yields empirically good results. The result is more stable for larger $\beta$, say $\beta =2$, compared to smaller $\beta$-values.

\begin{figure}[H]
    \centering
    \begin{subfigure}[t]{0.23\linewidth}
        \centering
        \includegraphics[width=\linewidth,  clip=true,trim=10px 10px 10px 10px]{eye_sigma_100.png}
        \caption{Classical inpainting as described above, $T=10$ and $\beta=0.25$}
    \end{subfigure}\;
    \begin{subfigure}[t]{0.23\linewidth}
        \centering
        \includegraphics[width=\linewidth,  clip=true,trim=10px 10px 10px 10px]{eye_sigma_100_T30.png}
        \caption{Classical inpainting as described above, $T=30$ and $\beta=0.25$}
    \end{subfigure}\;
    \begin{subfigure}[t]{0.23\linewidth}
        \centering
        \includegraphics[width=\linewidth,  clip=true,trim=10px 10px 10px 10px]{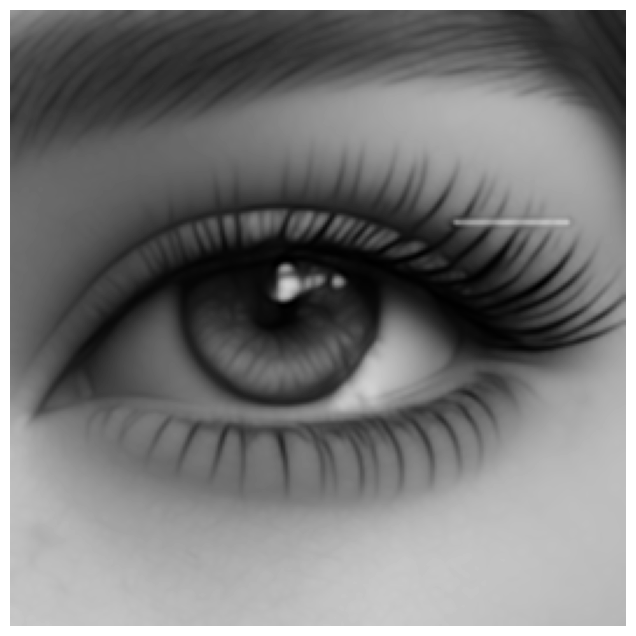}
        \caption{1 step of \WaxOn-\WaxOff, $\beta=0.25$ for \WaxOn and $\beta=5$ for \WaxOff}
    \end{subfigure}\;
    \begin{subfigure}[t]{0.23\linewidth}
        \centering
        \includegraphics[width=\linewidth,  clip=true,trim=10px 10px 10px 10px]{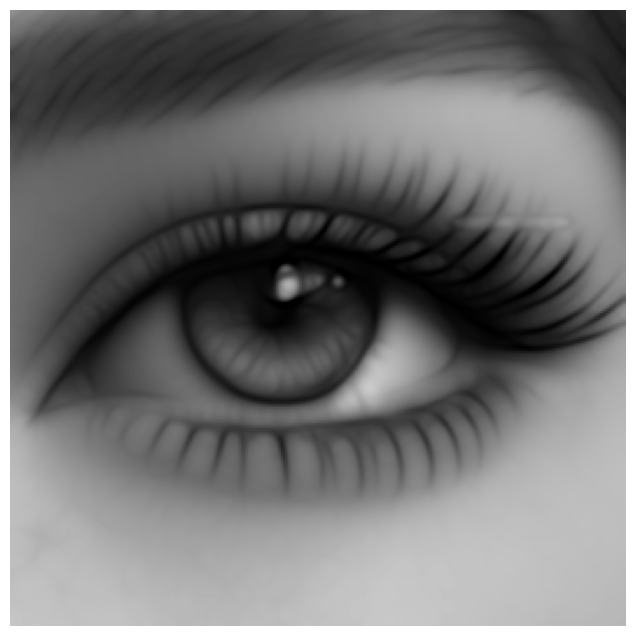}
        \caption{3 steps of \WaxOn-\WaxOff, $\beta=0.25$ for \WaxOn and $\beta=5$ for \WaxOff}
    \end{subfigure}
    \caption{The usual algorithm (our implementation) is applied to obtain (a) and (b). One iteration of \WaxOn-\WaxOff for small $T_2$ produces (c) while multiple iterations of \WaxOn-\WaxOff are sequentially applied to produce (d) which achieves a similar level of restoration to (b) but maintains an overall sharper image.}
\end{figure}

\begin{remark} There are several suggested methods for stabilizing solutions for the inverse heat conduction, which might increase the time we could apply \WaxOff. We refer to \cite{weber1981analysis,wang2010numerical,beck1970nonlinear, liu1996stability,beck1996comparison,alifanov1995extreme} as examples of such methods.
\end{remark}

\subsection{Unsharp masking}
\label{sec:WWunsharp}
The need to recover sharp images from blurred ones has been of interest in photography long before the invention of digital computers. In a photographic darkroom, this result can be physically achieved by copying the original glass-plate negative of the image, blurring it intentionally, and producing a scaled negative of it. If the two glass plates are now stacked one in front of the other and light is passed through both, the resulting image will see low-frequency information reduced while high-frequency information (acutance) enhanced.

In digital image processing, where glass-plate images are now arrays (matrices) and blurring is convolution with a normal distribution $G$ of mean $\mu=0$ and standard deviation $s$, such technique takes the form 
\[I \mapsto I+C(I- I\ast G_s)\]
where $I$ is the digital image, $G_s$ is the normal distribution, and $C\in\mathbb{R}$ is the sharpening factor \cite{Gonzalez2008}.
\begin{figure}[H]
    \centering
    \begin{subfigure}[t]{0.23\textwidth}
    \includegraphics[width=\textwidth]{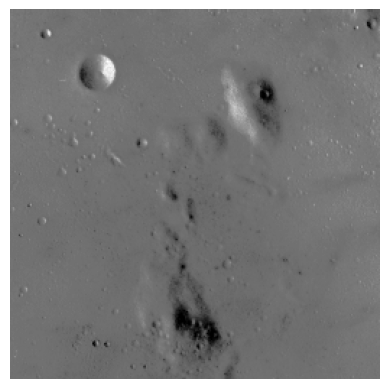}
    \caption{Original image, courtesy of \cite{vanderWalt2014}}
    \end{subfigure}
    \begin{subfigure}[t]{0.23\textwidth}
    \includegraphics[width=\textwidth]{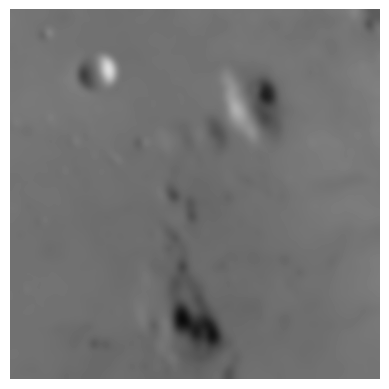}
    \caption{Gaussian blur with $\sigma=5$ is applied}
    \end{subfigure}
    \begin{subfigure}[t]{0.23\textwidth}
    \includegraphics[width=\textwidth]{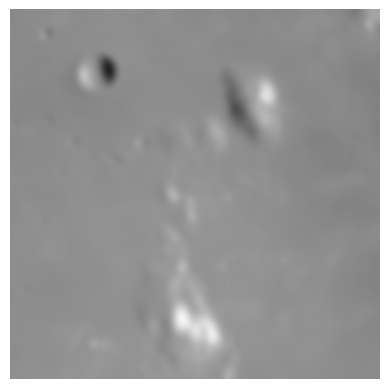}
    \caption{Negative of the blurred image}
    \end{subfigure}
    \begin{subfigure}[t]{0.23\textwidth}
    \includegraphics[width=\textwidth]{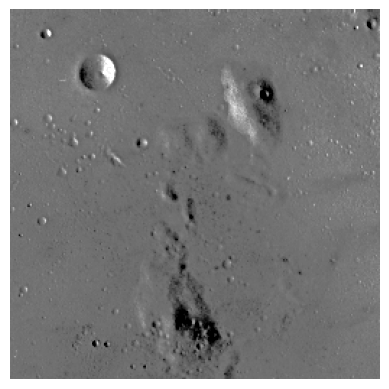}
    \caption{Sharpened image with $C=1$}
    \end{subfigure}
    \caption{Example of usage of the unsharp filter applied to a low-contrast image of the surface of the moon}
  \label{fig:unsharp_filter}
\end{figure}

If we consider kernels of size 3, this takes the form of a convolution with matrices of the form
\[\begin{bmatrix}0&0&0\\0&1&0\\0&0&0\end{bmatrix} + C\left(\begin{bmatrix}0&0&0\\0&1&0\\0&0&0\end{bmatrix} -\frac{1}{C}\begin{bmatrix}0&1&0\\1&1&1\\0&1&0\end{bmatrix} \right) = \begin{bmatrix}0&-1&0\\-1&C&-1\\0&-1&0\end{bmatrix} \]
The case $C=5$ is the usual unsharp filter of dimension 3, commonly used in image processing applications.

It comes quite naturally to consider an extension of unsharp masking to other domains, such as $\SE(2)$, to produce a curvature-sensitive sharpening filter aimed at enhancing digital images according to the mechanisms of the visual cortex V1.

The effect of the undesired blurring is therefore the solution at time $T$ to the Cauchy problem
\[\tilde \Delta_\beta= X_3^2 + \beta X_2^2 \qquad
\begin{cases}
    \partial_t u = \tilde \Delta_\beta u,\\
    u(0,x,y,\theta) = \tilde I (x,y,\theta)
\end{cases}\]
Denote such solution by $\tilde I_T (x, y, \theta) = u(T,x,y,\theta)$. Then the unsharp masking in $\SE(2)$ takes the form $\tilde I + C(\tilde I -  \tilde I_T)$ for a sharpening factor $C$.

\begin{figure}[H]
    \centering
    \begin{subfigure}[t]{0.3\linewidth}
        \centering
        \includegraphics[width=\linewidth]{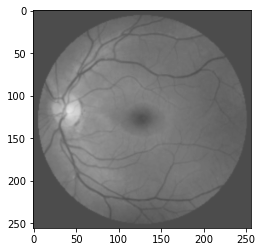}
        \caption{Original image courtesy of \cite{Hggstrm2014}}
    \end{subfigure}\;
    \begin{subfigure}[t]{0.3\linewidth}
        \centering
        \includegraphics[width=\linewidth]{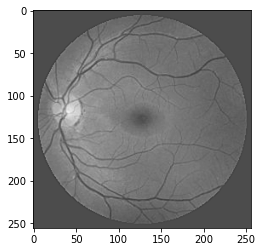}
        \caption{$\mathbb R^2$ classical unsharp filter}
    \end{subfigure}\;
    \begin{subfigure}[t]{0.3\linewidth}
        \centering
        \includegraphics[width=\linewidth]{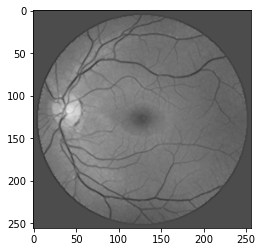}
        \caption{Proposed filter on $\SE(2)$}
    \end{subfigure}
    \caption{Retinal image (a) sharpened using the classical unsharp filter over $\mathbb R^2$ (b) and using the proposed sharpening method, after projection (c).}
\end{figure}

Defining an extension of the unsharp filter over $\SE(2)$ rather than $\mathbb R^2$ allows us to work with a lifted image in its natural domain rather than alternating projections and lifts. Combining the orientation-sensitive diffusion equation with unsharp filtering provides an effective tool to enhance the contrast of particularly predisposed images such as retinal scans, or used as a preprocessing step in a more complex pipeline.

\begin{figure}[H]
    \centering
    \begin{subfigure}[t]{0.19\linewidth}
        \centering
        \includegraphics[width=\linewidth]{retinal_original.png}
        \caption{Original image, courtesy of \cite{Hggstrm2014}}
    \end{subfigure}
    \begin{subfigure}[t]{0.19\linewidth}
        \centering
        \includegraphics[width=\linewidth]{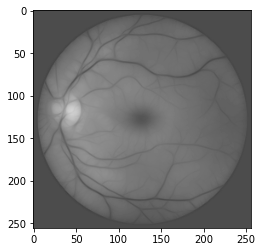}
        \caption{$C=0.5$}
    \end{subfigure}
    \begin{subfigure}[t]{0.19\linewidth}
        \centering
        \includegraphics[width=\linewidth]{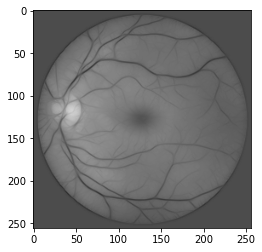}
        \caption{$C=1$}
    \end{subfigure}
    \begin{subfigure}[t]{0.19\linewidth}
        \centering
        \includegraphics[width=\linewidth]{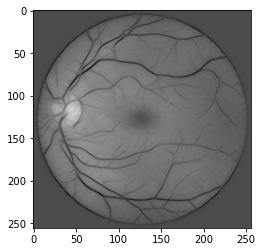}
        \caption{$C=1.5$}
    \end{subfigure}
    \begin{subfigure}[t]{0.19\linewidth}
        \centering
        \includegraphics[width=\linewidth]{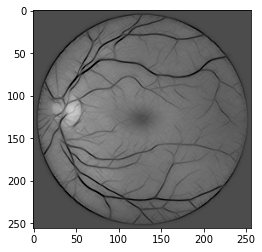}
        \caption{$C=2$}
    \end{subfigure}
    \caption{The original image (a) processed under $\Delta_\beta = X_1^2 + \beta^2 X_2^2$ and sharpened with varying coefficient for $C$ (subfigures b,c,d,e)}
\end{figure}

\subsection{\WaxOn-\WaxOff in the AHE algorithm}
Boscain, Chertoviskih, Gauthier, Prandi, and Remizov have introduced the AHE algorithm  \cite{Boscain2018-qd} as a strikingly powerful algorithm for image restoration when dealing with a variation of the problem in which the position of the corruption is known. The AHE algorithm is composed of 4 steps: simple averaging, strong diffusion, advanced averaging, and weak diffusion. Simple averaging uses the known mask to repeatedly fill the boundary of the corrupted portion of the image (1 pixel boundary at a time) with a local average on the non-corrupted portion of the image, proceeding in a way that recalls the BFS algorithm. Both strong diffusion and weak diffusion correspond to diffusion with varying (positive) coefficients $\Delta = a(x,y)X_1^2 + b(x,y)X_2^2$ where the coefficients in the weak diffusion are smaller than in the strong diffusion. Advanced averaging performs an average between the original image and the strongly diffused one, sharpening the image but also reintroducing the "mosaic effect" that was attenuated during strong smoothing. For more details consult \cite{Boscain2018-qd}.

\begin{algorithm}[H]
	\caption{AHE as presented in \cite{Boscain2018-qd}}\label{AHE}
	\begin{algorithmic}[1]
		\Procedure{AHE}{$I, M,T_1,T_2$}
		\State $I \gets FillMask(I,M)$ \Comment{Simple averaging as in \cite{Boscain2018-qd}}
		\State $J\gets Lift(I)$\Comment{Lift}
		\State $J\gets StrongDiffusion(J,I, T_1)$
		\State $I \gets Proj(J)$\Comment{Projection}
		\State $I\gets AdvancedAvg(I,M)$\Comment{Advanced averaging as in \cite{Boscain2018-qd}}
		\State $J\gets Lift(I)$\Comment{Lift}
		\State $J\gets WeakDiffusion(J,I, T_2)$
		\State $I \gets Proj(J)$\Comment{Projection}
		\State \textbf{return} $I$\Comment{The reconstructed image is $I$}
		\EndProcedure
	\end{algorithmic}
\end{algorithm}

We propose an enriched version of this algorithm which produces a sharper final result with higher contrast. 

\begin{algorithm}[H]
	\caption{Modified AHE with WaxOn-WaxOff}\label{AHEWax}
	\begin{algorithmic}[1]
		\Procedure{ModifiedAHE}{$I, M, n,T_1,T_2, T_3, T_4,sf$}
		\State $I \gets FillMask(I,M)$ \Comment{Simple averaging as in \cite{Boscain2018-qd}}
		\State $i\gets 0$
		\For{$i < n$}\Comment{Repeat WaxOn-WaxOff $n$ times}
		\State $J\gets Lift(I)$\Comment{Gaussian lift}
		\State $J\gets StrongWaxOnWaxOff(J,I, T_1, T_2)$ \Comment{As in \ref{wax_on_wax_off}, with varying coefficients}
		\State $I \gets Proj(J)$\Comment{Projection}
		\State $I\gets AdvancedAvg(I,M)$\Comment{Advanced averaging as in \cite{Boscain2018-qd}}
		\State $J\gets Lift(I)$\Comment{Gaussian lift}
		\State $J\gets WeakWaxOnWaxOff(J,I, T_3, T_4)$ \Comment{As in \ref{wax_on_wax_off}, with varying coefficients}
		\State $I \gets Proj(J)$\Comment{Projection}
		\State $I \gets Sharpen(I,sf)$
		\State $i \gets i+1$
		\EndFor\label{euclidendwhile}
		\State \textbf{return} $I$\Comment{The reconstructed image is $I$}
		\EndProcedure
	\end{algorithmic}
\end{algorithm}

In our implementation, the images are lifted via Gaussian lift and subsequently, the \WaxOn-\WaxOff procedure is applied, with $SE(2)$-unsharp used as \WaxOff, together with advanced averaging as introduced in \cite{Boscain2018-qd}. These two operations are performed repeatedly removing at each step the contour of the corrupted portion from the mask according to 4-point connectivity. The procedure is run until the mask is exhausted. In this way the images after smoothing become progressively sharper, as the mask is reduced step by step and advanced average is performed on a progressively decreasing number of pixels. 

\begin{figure}[H]
	\centering
    \hfil
    \begin{subfigure}[t]{0.23\linewidth}
        \centering
        \includegraphics[width=\linewidth,  clip=true,trim=10px 10px 10px 10px]{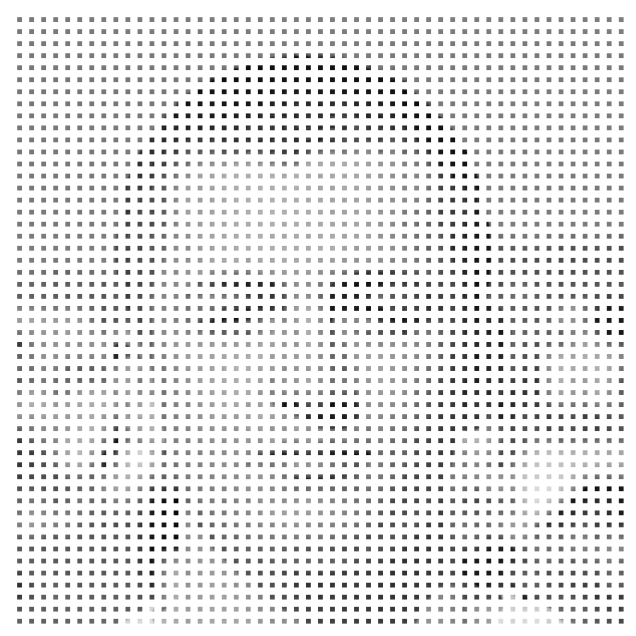}
        \caption{Masked image}
    \end{subfigure}\;
    \begin{subfigure}[t]{0.23\linewidth}
        \centering
        \includegraphics[width=\linewidth,  clip=true,trim=10px 10px 10px 10px]{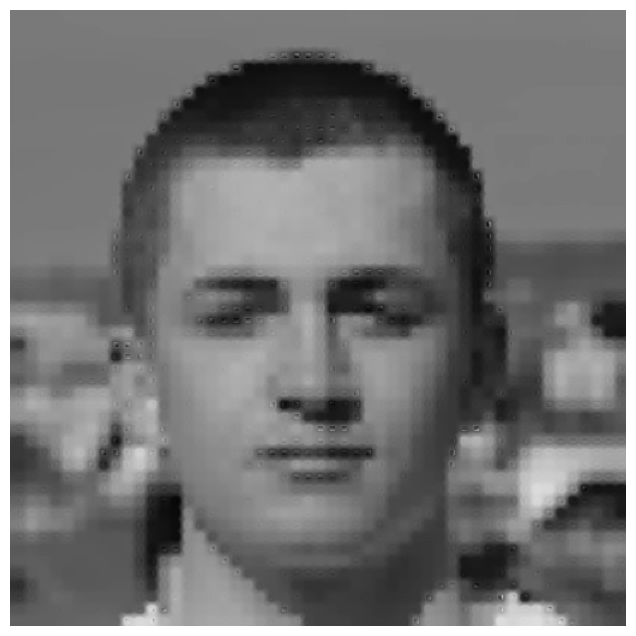}
        \caption{Masked region is filled according to \cite{Boscain2018-qd}}
    \end{subfigure}\;
    \begin{subfigure}[t]{0.23\linewidth}
        \centering
        \includegraphics[width=\linewidth,  clip=true,trim=10px 10px 10px 10px]{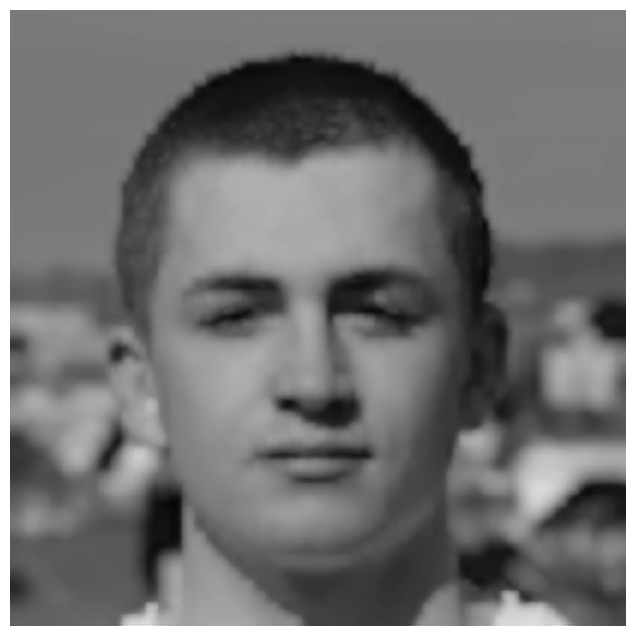}
        \caption{Original AHE}
    \end{subfigure}\;
    \begin{subfigure}[t]{0.23\linewidth}
        \centering
        \includegraphics[width=\linewidth,  clip=true,trim=10px 10px 10px 10px]{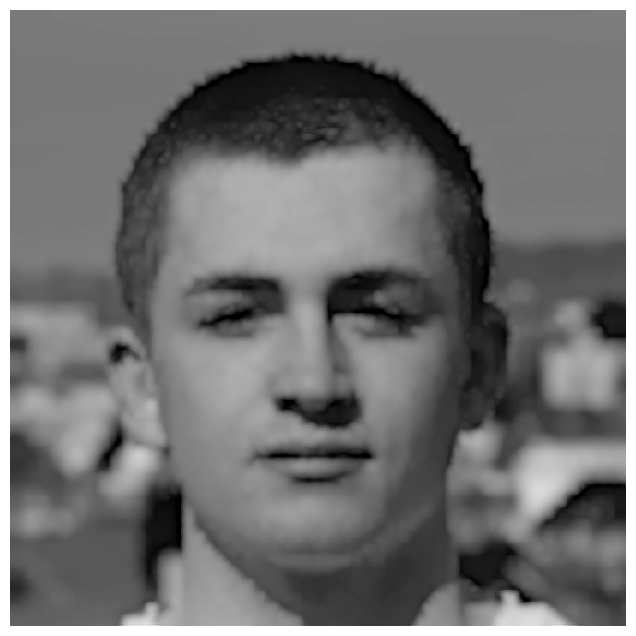}
        \caption{Our AHE}
    \end{subfigure}
    \hfil
    \newline
    \hfil
    \begin{subfigure}[t]{0.23\linewidth}
        \centering
        \includegraphics[width=\linewidth,  clip=true,trim=10px 10px 10px 10px]{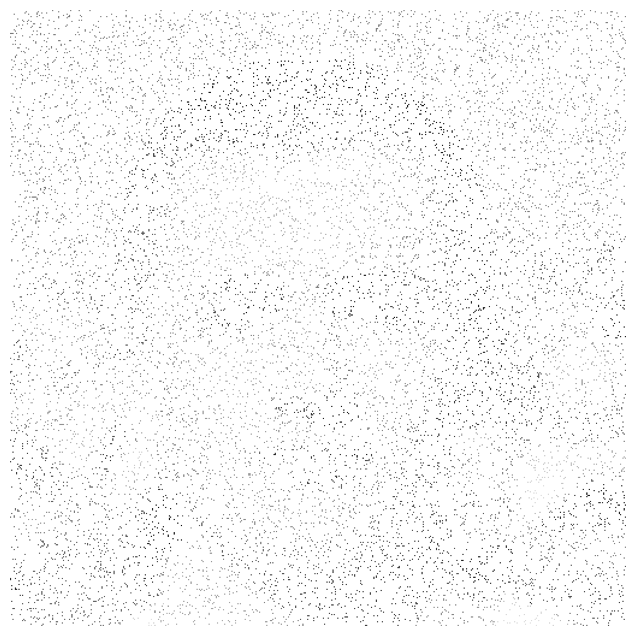}
        \caption{Masked image with 95\% uniform corruption}
    \end{subfigure}\;
    \begin{subfigure}[t]{0.23\linewidth}
        \centering
        \includegraphics[width=\linewidth,  clip=true,trim=10px 10px 10px 10px]{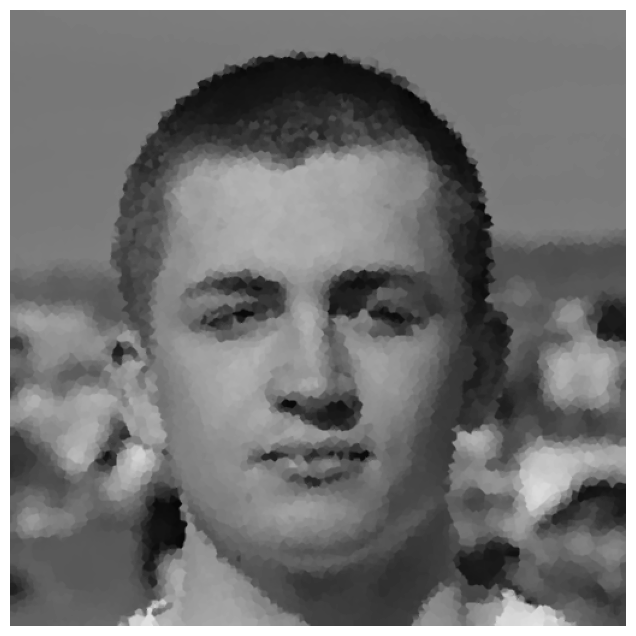}
        \caption{Masked region is filled according to \cite{Boscain2018-qd}}
    \end{subfigure}\;
    \begin{subfigure}[t]{0.23\linewidth}
        \centering
        \includegraphics[width=\linewidth,  clip=true,trim=10px 10px 10px 10px]{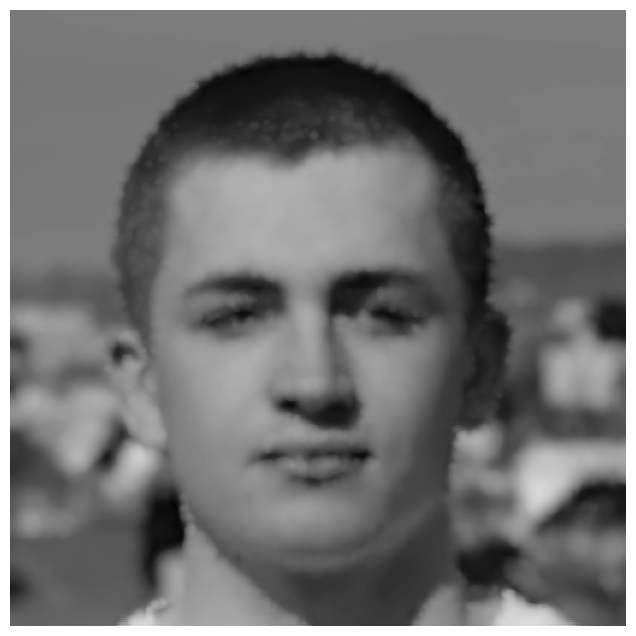}
        \caption{Original AHE}
    \end{subfigure}\;
    \begin{subfigure}[t]{0.23\linewidth}
        \centering
        \includegraphics[width=\linewidth,  clip=true,trim=10px 10px 10px 10px]{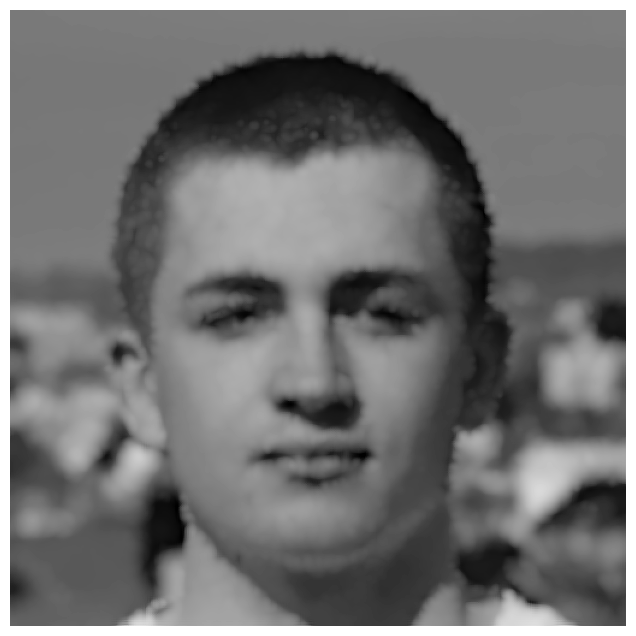}
        \caption{Our AHE}
    \end{subfigure}
    \hfil
    \caption{Example of image restoration with a known mask using AHE algorithm enriched by Gaussian lift and \WaxOn-\WaxOff. The last column shows a result with a higher contrast, due to the unsharp filter applied in the procedure.}
    \label{}
\end{figure}

\appendix
\renewcommand{\thesection}{\Alph{section}} 
\makeatletter
\renewcommand\@seccntformat[1]{\appendixname\ \csname the#1\endcsname.\hspace{0.5em}}
\makeatother

\section{Geometric preliminaries} \label{sec:Prelim}
\subsection{Sub-Riemannian manifolds} \label{sec:SRmanifold}
To introduce the reader who is unfamiliar with sub-Riemannian geometry we give here an introduction of the main concepts and results from the field. For a more in-depth study on the subject, we redirect to \cite{Montgomery2006}. 

\begin{definition}
	A sub-Riemannian manifold is a triplet $(M, \mathcal{H}, g)$ with $M$ being a connected manifold, $\mathcal H \subset TM$ a linear subbundle and $g=\langle\cdot, \cdot\rangle$ a fiber-metric defined on the subbundle~$\calH$.
\end{definition}
We call $\mathcal{H}\subset TM$ in this definition the \emph{horizontal distribution}. A sub-Riemannian manifold can be considered as a limiting case of a Riemannian manifold where the distances of vectors outside of $\calH$ approach infinity. Curves $\gamma:[a,b] \to M$ with a finite length will then need to be \emph{a horizontal curve}: an absolutely continuous curve satisfying $\dot \gamma(t) \in \calH_{\gamma(t)}$ for almost every $t$. For such a curve, we can define its length by
$$\length(\gamma) = \int_a^b \langle \dot \gamma(t), \dot \gamma(t) \rangle^{1/2} \, dt.$$
We can then also introduce the corresponding sub-Riemannian distance by
$$d_g(x,y) = \inf \left\{ \length(\gamma) \; \Big| \; \begin{array}{c}
\text{$\gamma:[a,b]\to M$ horizontal} \\ \gamma(a) = x, \gamma(b) = y \end{array}
\right\}.$$
In general, there might not be any curve connecting a point $x$ and $y$, meaning that the distance above will be infinite. It is therefore typical to require the horizontal bundle $\calH$ to be bracket-generating.
Let $\mathfrak{X}_{\calH}$ be all vector fields taking values in the subbundle $\calH$. We then define $\hat{\frakX}_{\calH} \supseteq \frakX_{\calH}$ as the space of all vector fields generated by those in $\frakX_{\calH}$ and their iterated Lie brackets. In other words
$$\hat{\frakX}_{\calH} = \spn \left\{ [X_{i_1}, [X_{i_2} ,[ \cdots [X_{l-1}, X_l] ] \cdots ]] \; | \; X_{i_j} \in \frakX_{\calH}, l = 1,2, 3, \dots, \right\},$$
where we interpret the case $l =1$ simply as the vector field $X_{i_1}$ itself.
\begin{definition}
We say that $\calH$ is bracket-generating if for every $x \in M$,
$$T_xM = \{ X(x) \, : \, X \in \hat{\frakX}_{\calH} \}.$$
\end{definition}
In other words, $\calH$ is bracket generating if we can make a partial derivative in any direction we want by combining directions in $\calH$. By the Chow-Rashevski\"i theorem \cite{Ras38,Cho39}, any two points in a sub-Riemannian manifold can be connected by a horizontal curve if $\calH$ is bracket-generating. Furthermore, $d_g$ will be a well-defined metric distance that has the same open sets as on the original manifold.
\bibliographystyle{abbrv}
\bibliography{bibliography_clean}

\end{document}